\documentclass[11pt]{article}
\pdfoutput=1

\usepackage{fullpage}
\usepackage{nicefrac}
\usepackage{listings}
\usepackage{mathtools}
\usepackage{fancyheadings}
\usepackage{hyperref,etoolbox}

\usepackage[margin=1in]{geometry}
\usepackage[utf8]{inputenc} 
\usepackage[T1]{fontenc}    
\usepackage{hyperref}       
\usepackage{subcaption}
\usepackage{graphicx}
\usepackage{floatrow}
\usepackage{booktabs}       
\usepackage{amsfonts}       
\usepackage{nicefrac}       
\usepackage{microtype}      

\usepackage[multidot,space]{grffile}

\usepackage{color}
\usepackage{graphicx}
\usepackage{amsmath,amssymb,amsfonts,bm,url,dsfont}

\usepackage{natbib}
\usepackage{float}
\usepackage{rotating}
\usepackage{subfloat}

\newtoggle{annals}
\togglefalse{annals}

\newcommand{\BlackBox}{\rule{1.5ex}{1.5ex}}  
\newenvironment{proof}{\par\noindent{\bf Proof\ }}{\hfill\BlackBox\\[2mm]}

 \usepackage[vlined,ruled,linesnumbered]{algorithm2e}
\usepackage{framed}

\newcommand{\calI}{\mathcal{I}}

\newcommand{\calX}{\mathcal{X}}

\newcommand{\calA}{\mathcal{A}}
\newcommand{\calB}{\mathcal{B}}

\newcommand{\calT}{\mathcal{T}}

\newcommand{\Ibar}{\overline{I}}
\newcommand{\underN}{\underline{N}}
\newcommand{\Int}{\mathrm{Int}}

\newcommand{\bin}{\mathsf{b}}
\newcommand{\stopt}{\tau}
\newcommand{\stoptnn}{\tau^{\mathrm{}}}

\newcommand{\fpoint}{f^{\mathrm{pnt}}}
\newcommand{\fsec}{f^{\mathrm{sec}}}

\newcommand{\class}{\mathcal{F}_{\mathrm{conv}}}
\newcommand{\fshift}{f_{\leftarrow t}}
\newcommand{\supp}{\mathrm{supp}}

\newcommand{\noise}{w}
\newcommand{\fls}{\widehat{f}_{\mathrm{LS}}}
\newcommand{\fst}{f_{\star}}

\newcommand{\linf}{L_{\infty}}
\newcommand{\ltwo}{L_{2}}

\newcommand{\stopp}{\boldsymbol{\tau}}

\newcommand{\Lamavg}{\Lambda_{\mathrm{avg}}}
\newcommand{\Lammax}{\Lambda_{\mathrm{max}}}

\newcommand{\calL}{\mathcal{L}}
\newcommand{\calG}{\mathcal{G}}
\newcommand{\calF}{\mathcal{F}}

\newcommand{\Alg}{\mathsf{Alg}}

\newcommand{\kl}{\mathrm{kl}}

\newcommand{\KL}{\mathrm{KL}}

\newcommand{\R}{\mathbb{R}}
\newcommand{\I}{\mathbb{I}}
\newcommand{\Exp}{\mathbb{E}}

\newcommand{\fhat}{\widehat{f}}

\newcommand{\Ibad}{I_{\mathrm{bad}}}

\newcommand{\tbad}{t_{\mathrm{bad}}}

\newcommand{\eps}{\varepsilon}

\renewcommand{\P}{\mathbf{P}}

\renewcommand{\Pr}{\mathbb{P}}

\def\E{\mathbb{E}}

\def\1{\mathbf{1}}
\def\P{\mathbb{P}}
\def\R{\mathbb{R}}

\def\calB{\mathcal{B}}
\def\calG{\mathcal{G}}

\def\@Aboxed#1&#2\ENDDNE{%
  \settowidth\@tempdima{$\displaystyle#1{}$}%
  \addtolength\@tempdima{\fboxsep}%
  \addtolength\@tempdima{\fboxrule}%
  \global\@tempdima=\@tempdima
  \kern\@tempdima
  &
  \kern-\@tempdima
  \fcolorbox{red}{yellow}{$\displaystyle #1#2$}
}

\newcommand{\Egood}{\mathcal{E}_{\mathrm{good}}}
\newcommand{\unifsim}{\overset{\mathrm{unif}}{\sim}}

\newcommand{\Npack}{N_{\mathrm{pck}}}
\newcommand{\xpack}{x^{(\mathrm{pck})}}
\newcommand{\Xpack}{\calX^{(\mathrm{pck})}}
\newcommand{\dyads}{\vec{s}}

\renewcommand{\Pr}{\mathbb{P}}

\newtheorem{theorem}{Theorem}[section]
\newtheorem{definition}{Definition}
\newtheorem{proposition}[theorem]{Proposition}
\newtheorem{remark}{Remark}[section]
\newtheorem{lemma}[theorem]{Lemma}

\newtheorem{fact}[theorem]{Fact}

\usepackage{mathtools}
\DeclarePairedDelimiter\ceil{\lceil}{\rceil}

\newcommand{\tleft}{t_\mathrm{left}}
\newcommand{\tright}{t_\mathrm{right}}

\def\1{\mathbf{1}}
\def\E{\mathbb{E}}
\def\P{\mathbb{P}}

\def\Xcal{\mathcal{X}}

\def\Tcal{\mathcal{T}}
\def\Lcal{\mathcal{L}}
\def\Sec{\mathrm{Sec}}

\def\parents{\text{\sf parents}}

\def\xr{x_{r(I)}}
\def\xl{x_{l(I)}}
\def\xm{x_{m(I)}}

\def\rI{{r(I)}}
\def\lI{{l(I)}}
\def\mI{{m(I)}}

\newcommand{\deltil}{\delta^{\mathrm{pnt}}}

\renewenvironment{quote}
  {\list{}{\rightmargin=.5cm \leftmargin=.5cm}%
   \item\relax}
  {\endlist}

\begin{document}
\title{\Large \bf Adaptive Sampling for Convex Regression}

\author{Max Simchowitz\thanks{UC Berkeley, CA. msimchow@berkeley.edu. } \and
Kevin Jamieson\thanks{University of Washington, WA. jamieson@cs.washington.edu }
\and
Jordan W. Suchow\thanks{UC Berkeley, CA.  suchow@berkeley.edu.} 
\and Thomas L. Griffiths\thanks{UC Berkeley, CA. tom\_griffiths@berkeley.edu}}

\maketitle{}
\begin{abstract}
In this paper, we introduce the first principled adaptive-sampling procedure for learning a convex function in the $L_\infty$ norm, a problem that arises often in the behavioral and social sciences. We present a function-specific measure of complexity and use it to prove that, for each convex function $f_{\star}$, our algorithm nearly attains the information-theoretically optimal, function-specific error rate. We also corroborate our theoretical contributions with numerical experiments, finding that our method substantially outperforms passive, uniform sampling for favorable synthetic and data-derived functions in low-noise settings with large sampling budgets. Our results also suggest an idealized ``oracle strategy'', which we use to gauge the potential advance of \emph{any} adaptive-sampling strategy over passive sampling, for any given convex function.

\end{abstract}





\section{Introduction}
Many functions that model individual economic utility, the output of manufacturing processes, and natural phenomena in the social sciences are either convex or concave.
For example, convex functions are used to model utility functions that exhibit {\em temporal discounting}, a classic effect in behavioral economics where people value immediate rewards over delayed rewards~\citep{frederick2002time, green2004discounting}.
To measure such curves, it is common practice to manipulate a variable (e.g., delay) over a fixed, uniformly spaced grid of $\approx$ \emph{design points} 5 points \citep{fisher1937design}, collect many repeated trials of data, and fit a function of assumed parametric form (e.g., exponential or hyperbolic) using maximum likelihood estimation.
This approach can be brittle to model mismatch when the true function $f$ lies outside the assumed class of functions. Moreover, non-linear parametric families can introduce challenges for constructing faithful and accurate confidence intervals when interpolating the estimator between measured design points.

Non-parametric convex regression (c.f. \cite{dumbgen2004consistency}) corrects for the shortcomings of parametric methods by making no assumptions other than that $f$ is convex. In additional to faithfully modeling a large class of functions, non-parametric methods can also be employed to construct error bars at any $x \in [0,1]$ (see \cite{cai2013adaptive}). 
Unfortunately, even with shape restrictions, non-parametric methods may require prohibitively many samples for practical use. 


In this paper, we propose a more parsimonious approach to non-parametric curve estimation by allowing the design points to be chosen sequentially and adaptively. Formally, we consider the problem of estimating an unknown convex function $\fst:[0,1] \to \R$ with an estimator $\fhat$ which is close in the  $L_\infty$ metric $\|\widehat{f} - f^*\|_\infty = \sup_{x}|\fst(x) - \widehat{f}(x)|$~. The estimator is constructed from sequential, noisy evaluations $y_1,\dots,y_{\stopp}$ from an oracle $F$ at design points $x_1,\dots,x_{\stopp}$,
\begin{align}\label{eqn:noisy_oracle}
y_t = F(x_t),~\text{where } F(x_t)  = \fst(x_t) + \noise_t,
\end{align} 
and where $\noise_t$ represents zero-mean noise. 
We let $\calF_t$ denote the filtration generated by the design points and measurements $(x_s,y_s)_{s=1}^t$ up to time $t$, 
and assume that the number of samples $\stopp$ is a stopping time with respect to $\{\calF_t\}$, where $\noise_t \big{|} \calF_{t-1}$ is zero mean, $\sigma^2$-subgaussian. We refer to measurement allocation strategies for which $x_{t+1}$ does not depend on $(x_s,y_s)_{s=1}^t$ as \emph{passive}, and adaptive sampling strategies for which $x_{t+1}$ may depend on $(x_s,y_s)_{s=1}^t$ as \emph{active}. 

Our main contributions are the following:
\begin{itemize}
	\item Inspired by~\cite{cai2013adaptive}, we introduce the \emph{local approximation modulus}, a 
	new measure of local curvature for convex functions, $\omega(\fst,x,\epsilon)$, and a function-specific complexity measure $\Lamavg(\fst,\epsilon) \approx \int_{0}^1 \omega(\fst,x,\epsilon)^{-1}dx$, called the \emph{average} approximation modulus. $\Lamavg(\fst,\epsilon)$ coincides with the \emph{average} curvature of $\fst$, up to logarithmic factors and endpoint considerations. 
	We prove a function-specific lower bound on the sample complexity of actively estimating any convex function $\fst$ to $\linf$-error $\epsilon$ that scales at least as $(1 + \frac{\sigma^2}{\epsilon^2})~\Lamavg(\fst,\epsilon)$, up to logarithmic factors. 
	\item The packing argument for constructing our lower bound explicitly describes a near-optimal, clairvoyant sampling allocation tailored to $\fst$; we call this the ``oracle allocation'' (Proposition~\ref{prop:oracle}).
	This allocation is instructive as an experimental benchmark when $\fst$ is known. 
	\item We introduce an active sampling procedure and an estimator $\fhat$ whose sample complexity for any \emph{particular} $f^*$ scales as $(1 + \frac{\sigma^2}{\epsilon^2}) \cdot \Lamavg(\fst,\epsilon)$ up to logarithmic factors, nearly matching our lower bounds. 
	%
	%
	\item We show that for passive designs (e.g., sampled evenly on a grid), the sample complexity necessarily scales as $(1 + \frac{\sigma^2}{\epsilon^2}) \cdot \Lammax(f,\epsilon)$, where $\Lammax(f,\epsilon) \approx \max_{x \in [0,1]} \omega(\fst,x,\epsilon)$ coincides with the $\emph{maximum}$ curvature.  We compare $\Lamavg(\fst,\epsilon)$ and $\Lammax(\fst,\epsilon)$ for many natural classes of functions, including quadratic functions, exponential curves, and $k$-piecewise linear functions. For $k$-piecewise linear functions, the $L_\infty$ error of our active algorithm scales no slower than $kn^{-1/2}\log n$, whereas passive designs scale no faster than $n^{-1/3}$ after $n$ evaluations (see Remark~\ref{rem:piecewiselinear}).   
\end{itemize}

Finally, we validate our theoretical claims with an empirical study using both synthetic functions and those derived from real data. 
We observe that in low-noise settings or when the sampling budget is large, active sampling can substantially outperform passive uniform sampling.
Moreover, our algorithm constitutes the first theoretically justified algorithm (passive or active) that guarantees uniform accuracy, even at the boundaries of the interval \citep{cai2013adaptive,dumbgen2004consistency}.
Even so, comparing the performance of our active algorithm to the oracle sampling strategy suggests room for modest but non-negligible improvements.

\subsection{Related Work}
\cite{castro2005faster} and~\cite{korostelev1999minimax} studied the minimax rates of active non-parametric regression, showing that active and passive learning attain the same minimax rates of convergence for Holder smooth classes, but that active learning achieves faster rates when the function is known to be well approximated by a piecewise-constant function.

Prior literature on convex and concave regression consider the \emph{passive design} case, where the design points do not depend on measurements. Typically, the design points are chosen to be uniformly spaced on the unit interval, that is, $x_i = \frac{1}{n-1}$ for $i=0,\dots,n-1$ \citep{dumbgen2004consistency}. If $\mathcal{F}$ is the set of Lipschitz, convex functions, then the $L_\infty$-norm $\|\fls  - \fst\|_{L_\infty} = \sup_{x \in [0,1]}|\fst(x) - \fls(x)|$ of the least squares estimator $\fls$ decreases like $(\log(n)/n)^{1/3}$, whereas if the convex function has Lipschitz gradients, the rate improves to $(\log(n)/n)^{2/5}$ \citep{dumbgen2004consistency}. 

Recent work by~\cite{guntuboyina2015global} and~\cite{chatterjee2016improved} has aimed at developing sharp errors bounds on the squared $\ltwo$-norm $\|\fhat - \fst\|_{\ltwo}^2 := \int_{x \in [0,1]} |\fhat(x) - \fst(x)|^2 dx$ of the least squares estimator, when samples are uniformly spaced on a grid. They show that even with this uniform allocation, the error $\|\fhat - \fst\|_{\ltwo}^2$ \emph{adapts} to the true regression functions $\fst$. For example,~\cite{chatterjee2016improved} and~\cite{bellec2018sharp} show that if $\fst$ is a $k$-piecewise linear function, then $\fls$ obtains the parametric error rate of $\|\fls - \fst\|_{\ltwo}^2 \le C k/n$. In a similar vein, \cite{cai2013adaptive} proves sharp confidence intervals for $\fst(x_0)$ for a fixed point $x_0 \in (0,1)$. 

Our work draws heavily upon~\cite{cai2013adaptive} (who in turn build on~\cite{dumbgen2003optimal}), whose aim was to characterize the function-specific sample complexity of estimating a convex $f$ at a given point in the interior of $[0,1]$, from uniform measurements. We extend these tools to characterize the complexity of estimating $f$ with uniform accuracy over the interval $[0,1]$, from measurements which may be chosen in an adaptive, function-dependent manner. We are thus able to obtain exceptionally granular, instance-specific results similar to those in the multi-arm bandit literature~\citep{kaufmann2016complexity}, and in recent work studying the local minimax sample complexity of convex optimization \citep{zhu2016local}.


\vspace{-5pt}
\section{Efficiently Learning a Convex Function\label{sec:cvx_prelim}}

We begin by establishing preliminary notation. 
The class of convex functions is denoted as $\class :=\{f:[0,1] \to \R| f( (1-\lambda)x + \lambda y) \leq (1-\lambda) f(x) + \lambda f(y),~ \forall x,y,\lambda\in[0,1]\}$.
For an interval $I = [a,b] \subseteq [0,1]$, define the left-, middle- and right-endpoints as $\xl = a, \xm = \frac{a+b}{2}, \xr=b$. We define the secant approximation of $f$ on an interval $I \subset [0,1]$ as  
\begin{eqnarray}
\Sec[f,I](x) = \frac{\xr-x}{\xr-\xl} f(\xl) + \frac{x-\xl}{\xr-\xl} f(\xr)~,
\end{eqnarray} 
and note that for a convex function, this approximation never underestimates $f$; that is, one has $\Sec[f,I](x) \ge f(x)$. We denote the error of the second approximation to $f$ on $I$ at the midpoint $x_m$ as 
\begin{align*}
\Delta(f, I) :=   \Sec[f,I](\xm) - f(\xm) = \frac{f(x_\lI)+f(x_\rI)}{2} - f(x_\mI)~.
\end{align*}
In addition, we overload notation so that for any $x,t$ such that $x \in [t,1-t]$, we have $\Delta(f,x,t) := \Delta(f,[x-t,x+t])$. 

We now state a remarkable fact about convex functions that is at the core of our analysis.
\begin{lemma}\label{lem:mainLemmaSimple}
For  convex $f$, $\displaystyle\Delta(f, I) \leq \max_{x \in I} \left\{\Sec[f,I](x) - f(x)\right\} \leq 2 \Delta(f, I)$.
\end{lemma}
Lemma~\ref{lem:mainLemmaSimple} is a special case of a more general lemma stated in Section~\ref{l_infty_lemma} that upper bounds the supremum of the secant approximation error by a constant using only a single point within the interval.
Convexity is critical to the proof of this lemma and such a property does not hold, for instance, on merely monotonic functions.
We remark that the first inequality is trivial, and the second inequality is tight in the sense that it is achieved by $f(x) = (1-x)^p$ on interval $I = [0,1]$ as $p \rightarrow \infty$.

The above observations motivate our strategy of approximating $f$ with secant approximations on disjoint intervals whose union is $[0,1]$.
The next definition relates the secant approximation error to the required sampling density.

\begin{definition}[Local Approximation Modulus]\label{Modulus} We define the $\epsilon$-approximation modulus of $f$ at a point $x \in [0,1]$ as the least $t$ such that the midpoint secant approximation to $f$ on $[x-t,x+t]$ has bias $\epsilon$:
\begin{align}\label{eqn:omega_def}
\omega(f,x,\epsilon) := \min\left\{t \in \left[0,\min\left\{x,1-x\right\}\right]: \Delta\left(f,x, t\right) \ge   \epsilon   \right\}.
\end{align}
\end{definition}
Note that $\omega(f,x,\epsilon) > 0$ for all $x \in (0,1)$, because convex functions are continuous on their domain.
Intuitively, $\omega(f,x,\epsilon)$ is the scale at which $f$ ``looks'' linear around some $x$, up to a tolerance $\epsilon$. 
Away from the endpoints $\{0,1\}$, smaller values of $\omega(f,x,\epsilon)$ correspond to larger complexities, because they imply that $f$ can only be approximated by a linear function on a small interval. But if $x$ is the near $\{0,1\}$, $\omega(f,x,\epsilon) \le \min\{x, 1-x\}$ will take small values, potentially overestimating the local complexity. We remedy this issue by defining the following left- and right-approximation points:
\begin{align*}
\tleft(f,\epsilon) &:= \inf \left\{t \le 1/2: \Delta\left(f,t,t\right) \ge \epsilon\right\}  \\
\tright(f,\epsilon) &= \inf \left\{t \le 1/2: \Delta\left(f,1-t,t\right) \ge \epsilon\right\} \, .
\end{align*}
Within $[\tleft(f,\epsilon), 1 - \tright(\epsilon)]$, we will show that the midpoint errors $\Delta(f,x,t)$ concisely describe how densely one would need to sample $f$ in the neighborhood of $x \in [0,1]$ in order to estimate it to the desired accuracy $\epsilon$ in the $L_\infty$-norm. Moreover, we show that it suffices to sample at constant number of design points on the end-intervals $[0,\tleft(f,\epsilon)]$ and $[1-\tright(\epsilon),1]$. 
At a high level, the main finding of this paper is as follows:
\begin{quote}\vspace{-2pt}
\emph{The sample complexity of learning a particular convex function $\fst$ up to $\epsilon$ accuracy in $L_\infty$ with \textbf{passive sampling} is parametrized by the \textbf{worst-case} approximation modulus}
\begin{eqnarray}
\Lammax(\fst,\epsilon) := 1 + \sup_{x \in [\tleft(\fst,\epsilon),1 - \tright(f,\epsilon)]} \omega(\fst,x,\epsilon)^{-1}.
\end{eqnarray}
\emph{In contrast, the sample complexity of \textbf{active sampling} algorithms is parametrized by the \textbf{average} approximation modulus}
\begin{align}
\Lamavg(\fst,\epsilon) := 1 + \int_{\tleft(\fst,\epsilon)}^{1-\tright(\fst,\epsilon)} \omega(\fst,x,\epsilon)^{-1} dx~.
\end{align}
\end{quote}
We emphasize that the algorithm presented in in this work guarantees accuracy on the whole interval $[0,1]$, whereas many passive algorithms pointwise and $\linf$ risk bounds \citep{,cai2013adaptive,dumbgen2004consistency} only guarantee accuracy on a strictly smaller sub-interval. 

\subsection{Examples} \label{sec:examples}
Explicit parameterizations of $\fst$ provide intuition for when active sampling is advantageous. In this section, we describe different scalings of $\Lammax$ and $\Lamavg$ for various $\fst$; 
later, in Remark~\ref{rem:piecewiselinear}, we explain how these scalings can imply substantial differences in sample complexity.

\subsubsection{Piecewise Linear Functions} Let $\fst$ be a Lipschitz, piecewise linear convex function with a constant number of pieces. 
It follows that $\omega(\fst,x,\epsilon)^{-1} \approx \min\{\tfrac{1}{\epsilon}, \tfrac{1}{d(x,\fst)}\}$ where $d(x,\fst)$ is the distance to the closest knot adjoining any two linear pieces of $\fst$. 
It follows that  $\Lammax(f,\epsilon) \approx \epsilon^{-1}$ whereas $\Lamavg(f,\epsilon) \approx \log(1/\epsilon)$.

\subsubsection{Bounded third-derivative:} 
Suppose $\sup_{x \in [0,1]}\fst'''(x)<\infty$. 
We may apply a Taylor series to find $\omega(\fst,x,\epsilon)^{-1} = \sqrt{\fst''(x)/2\epsilon}$ as $\epsilon \rightarrow 0$, which makes an explicit connection between the curvature of the function and the differences between $\Lamavg(\fst,\epsilon)$ and $\Lammax(\fst,\epsilon)$.
This suggests that if the function has areas of high but localized curvature such as $\fst(x) = 1-\sqrt{x}$ or $\fst(x) = \frac{1}{100}\log(1+\exp(-100(x-\frac{1}{2})))$ then the difference between $\Lamavg(\fst,\epsilon)$ and $\Lammax(\fst,\epsilon)$ can be as vast as $\log(1/\epsilon)$ versus $1/\epsilon$.

\subsubsection{Quadratic Functions:} Let $\fst(x) = \tfrac{1}{2} a x^2 + bx + c$ for some real coefficients $a,b,c$.
Ignoring the effect of endpoints, $\omega(\fst,x,\epsilon)^{-1} = \sqrt{\frac{ a}{2\epsilon}}$ for all $x$ due to the function having constant curvature, so $\Lamavg(\fst,\epsilon)=\Lammax(\fst,\epsilon)$.

\section{Main Results}
In this section, we state a formal upper bound obtained by Algorithm~\ref{AlgorithmNoise}, described in Section~\ref{sec:algorithm}. Algorithm~\ref{AlgorithmNoise} takes in a confidence parameter $\delta \in (0,1)$, as well as a second parameter $\beta > 0$ governing the degree to which the active sampling algorithm is `aggressive'; from simulations, we recommend setting $\beta = 1/2$. Lastly, at each round, Algorithm~\ref{AlgorithmNoise} maintains an estimator $\widehat{f}_t \in \class$, whose performance is characterized by the following theorem:

\begin{theorem}\label{thm:upperbound_samples_noisy}
Let $C > 0$ be a universal constant, and for $\fst \in \class$, $\delta \in (0,1/2)$ and $\beta > 0$, define $c_{\beta} := \tfrac{\beta/2}{2+\beta}$, and
\begin{align*}
\overline{N}_{\beta}(\fst,\epsilon,\delta):= C\Lamavg(\fst,c_{\beta} \epsilon) \, \max\left\{ 1, \tfrac{(1+\beta)^2\sigma^2}{\epsilon^2} \right\} \,  \log\left( \frac{1}{\delta} \, \Lamavg(\fst,c_{\beta} \epsilon)  \, \log \left(1 + \tfrac{\sigma(1+\beta)}{\epsilon}\right)\right)~.
\end{align*}
Then,  if Algorithm~\ref{AlgorithmNoise} is run with parameters $\delta$ and $\beta$, with access to an oracle \eqref{eqn:noisy_oracle}, the estimators $\widehat{f}_t \in \class$ and confidence estimates $\epsilon_t$ defined in Section~\ref{sec:noisy_sample_complexity} satisfy the following any-time guarantee:
\begin{align*}
\Pr_{\fst,\Alg}\left[ \|\widehat{f}_t - \fst\|_{\infty} \le \epsilon_t \le \epsilon  \text{ for all } t, \epsilon: t \ge \overline{N}(\fst,\epsilon,\beta,\delta)\right] \ge 1 - \delta.
\end{align*}
\end{theorem}
In the case of the default parameter setting $\beta = 1/2$, we find that, for a possibly larger universal constant $C'$,
\begin{align*}
\overline{N}_1(f,\epsilon,\delta) \le C'\Lamavg\left(f, \frac{\epsilon}{10}\right) \cdot \max\left\{ 1, \frac{\sigma^2}{\epsilon^2} \right\} \cdot  \log\left( \frac{1}{\delta} \cdot \Lamavg\left(f,\frac{\epsilon}{10}\right)  \cdot \log \left(1 + \frac{\sigma}{\epsilon}\right)\right)~.
\end{align*}
Up to constants and logarithmic factors, the sample complexity is dominated by the term $\Lamavg(f, \epsilon/10) \cdot \max\left\{ 1, \frac{\sigma^2}{\epsilon^2} \right\}$. Here $\frac{\sigma^2}{\epsilon^2}$ corresponds to the standard rate for estimating a scalar. The dependence on $\Lamavg(f,c_{\beta} \epsilon)$ captures the number of points required to estimate $f$ with a discretized proxy. 

To better understand why $\Lamavg$ is the appropriate quantity to consider, we now introduce a construction of local packings of $\class$, centered at a given $f \in \class$. Recall that $\Delta(f,I)$ denotes the error of the secant approximation to $f$ on the midpoint $\xm$ of $I$, constructed using the endpoints $\xl,\xr$. We note that if any algorithm, even an active one, does not measure $f$ on an interval $I$ for which $\Delta(f,I) \ge \epsilon$, then one cannot distinguish between $f$ and the alternative function $\tilde{f}_I := f(x) + \I(x \in I) ( \Sec[f,I](x) - f(x))$. Thus, a key step to showing that $\Lambda(f,\epsilon)$ approximately lower bounds the number of evaluations is to show that it approximately lower bounds the number of intervals $I$ for which $\Delta(f,I) \ge \epsilon$.
This is achieved in the following theorem proved in Section~\ref{sec:packing_proof}:
\begin{theorem}[Packing] \label{thm:packing}
Let $f\in \class$ be a convex function, $\epsilon>0$, and define
\begin{eqnarray}
\underline{N}(f,\epsilon) = \ceil{\frac{\Lamavg(f,\epsilon)}{4(1+\log(\omega_{\max}(f,\epsilon)/\omega_{\min}(f,\epsilon)))} - 2}
\end{eqnarray}
where $\omega_{\max}(f,\epsilon) := \max_{x \in [\tleft(f,\epsilon),1-\tright(f,\epsilon)]}\omega(f,x,\epsilon)$ and $\omega_{\min}$ is defined analogously. Then, there is an $\Npack(f,\epsilon) \ge \underN(f,\epsilon)$ such that the points $\{z_i\}_{i=1}^{\Npack(f,\epsilon) } \subset [0,1]$ such that the intervals 
\begin{eqnarray*}
I^{\epsilon}_i := [z_i - \omega(z_i,f,\epsilon),z_i + \omega(z_i,f,\epsilon)],~ i \in [\Npack(f,\epsilon) ]
\end{eqnarray*} 
have disjoint interiors, are contained in $[2\tleft(f,\epsilon),1-2\tright(f,\epsilon)]$ and satisfy $\Delta(f,I_i^{\epsilon}) = \epsilon$. Moreover, the interval endpoints overlap so that $x_{\ell(I_{i+1}^\epsilon)} = x_{r(I_{i}^{\epsilon})}$.
\end{theorem}
Note that $\Npack(f,\epsilon)$ corresponds to the actual size of the explict packing, and $\underN(f,\epsilon)$ is a computable lower bound on $\Npack(f,\epsilon)$. We now consider the class $\calG(f,\epsilon) \subset \class$ of alternative functions 
\begin{align}
\mathcal{G}_{f,\epsilon} := \Big\{ f(x) + \sum_{i=1}^{\Npack(f,\epsilon)} \bin_i \I\{ x \in I^{\epsilon}_i \}(\Sec[f,I^{\epsilon}_i](x)-f(x)) : \bin \in \{0,1\}^{\Npack(f,\epsilon)} \Big\}.\label{Local_Alternatives_Eq}
\end{align}
We observe that $f \in \mathcal{G}_{f,\epsilon} \subset \class$, and by definition, if $g_1 ,g_2 \in \calG_{f,\epsilon}$ are distinct, then $\|g_1-g_2\|_{\infty} \ge \epsilon$. 
In particular, given any set of points $\{x_i\}_{i=1}^{n} \subset [0,1]$ for $n < \Npack(f,\epsilon)$, then there exist two convex functions $g_1,g_2$ in $\calG_{f,\epsilon}$, such that $g_1(x_i) = g_2(x_i)$ for all $i$ and $\|g_1 - g_2\|_{\infty} \ge \epsilon$. 
In Section~\ref{sec:proof_of_lowerbound_noisy}, we formalize this argument to yield the following theorem:

\begin{theorem}\label{thm:lowerbound_noisy}
Fix an $\fst \in \class$, $\epsilon >0$, and $\delta \in (0,1/3)$.
Let $\underline{N}(\fst,\epsilon)$ be as in Lemma~\ref{thm:packing}, and let and $\mathcal{G}_{f,\epsilon}$ be as given by Equation~\eqref{Local_Alternatives_Eq}. Let $\Alg$ be any active algorithm that returns an estimator $\widehat{f}$ at a stopping time $\stopt$, and satisfies the correctness guarantee
\begin{align}\label{eq:correctness}
\forall g \in  \mathcal{G}_{\fst,2\epsilon},~\Pr_{\Alg,g}[\|\widehat{f} - g\|_{\infty} < \epsilon] \ge 1 - \delta.
\end{align}
Then the stopping time $\stopt$, under observations from $f$, is lower bounded by 
\begin{align*}
\E_{\fst,\Alg}[\stopt]  \gtrsim \underline{N}(\fst,2\epsilon) \cdot \max\left\{ 1, \frac{\sigma^2}{2\epsilon^{2}}  \log(1/\delta) \right\}~,
\end{align*}
and the average sample complexity over $\mathcal{G}_{f,2\epsilon}$ is at least
\begin{align}\label{eqn:minimax_avg_lower}
\frac{1}{\left|\mathcal{G}_{\fst,2\epsilon}\right|}\sum_{g \in \mathcal{G}_{\fst,2\epsilon}}\E_{g,\Alg}[\stopt]  \gtrsim \underline{N}(\fst,2\epsilon) \cdot \max\left\{ 1, \frac{\sigma^2}{2\epsilon^{2}} \cdot \log(\underline{N}(\fst,2\epsilon)/\delta) \right\}.
\end{align}
The above bounds hold when $\underN(\fst,\cdot)$ is replaced by $1 \vee \Npack$. 
\end{theorem}
\begin{remark}
The additional logarithmic factor that arises in~\eqref{eqn:minimax_avg_lower} is due to the fact that estimating a function $g \in \mathcal{G}_{f,2\epsilon}$ to $L_\infty$-error $\epsilon$ corresponds to correctly performing $\underline{N}(f,2\epsilon)$ simultaneous hypothesis tests, regarding the value of $g$ on each of the intervals $I_i^{\epsilon}$. However, for any fixed $g \in \calG_{f,2\epsilon}$ (and, in particular, $f = g$), one can devise an algorithm that does not suffer this logarithmic factor by `biasing' the algorithm towards that function.
\end{remark}

In addition to providing a lower bound, the packing of Theorem~\ref{thm:packing} defines a near-optimal covering as well, in the sense that it defines a sampling allocation that can be used to \emph{test the hypothesis} that, for a given $\fst$, $H_0: \{f = \fst\}$ versus $H_1: \|f - \fst\|_{\infty} = \Omega(\epsilon)$. Formally, we have the following:
\begin{proposition}\label{prop:oracle} For every function $\fst$, $\epsilon > 0$ and $\delta \in (0,1)$, there exists a 
\begin{align*}
T \lesssim (1 + \frac{\sigma^2}{\epsilon^2})\cdot (1 \vee \Npack(\fst,\epsilon)) \cdot \log((1 \vee \Npack(\fst,\epsilon)/\delta), 
\end{align*}
a deterministic sampling allocation $\Xpack := \{\xpack_1,\dots,\xpack_T\}$, and a test function $\psi \in \{0,1\}$ constructed from the allocation $\Xpack$ such that
\begin{align*}
\forall f \in \class: \|f - \fst\| \ge 10\epsilon, \quad \Pr_{\fst}[\psi \ne 0] + \Pr_{f}[\psi \ne 1] \le \delta~.
\end{align*}
\end{proposition}
The design $\Xpack$ is explicitly constructed in Section~\ref{sec:packing_proof} by augmenting the $\Npack(\fst,\epsilon)$-intervals in Theorem~\ref{thm:packing} with at most three additional intervals to ensure coverage of all of $[0,1]$. Crucially, we made use of the fact that intervals $I_i^{\epsilon}$ share endpoints, and have secant error $\Sec[\fst, I_i^{\epsilon}]$ exactly equal to $\epsilon$. In light of Theorem~\ref{thm:lowerbound_noisy}, we see that the design $\Xpack$ is optimal for verifying that $f = \fst$, up to scaling $\epsilon$ by constant factors. For this reason, we refer to this construction as the \emph{oracle allocation} since it precisely characterizes the optimal sampling allocation taken if one \emph{knew} $\fst$. In general, this allocation may be too optimistic, since an algorithm which does not know the true $\fst$ cannot choose this allocation a fortiori.

\subsection{Comparison between Upper and Lower Bounds}
For the purpose of comparing upper and lower bounds, we will consider running Algorithm~\ref{AlgorithmNoise} with the setting $\beta = 1$; any constant $\beta$ bounded away from zero will yield qualitatively similar results. We find that the upper bound of Theorem~\ref{thm:upperbound_samples_noisy} and lower bound of Theorem~\ref{thm:lowerbound_noisy} nearly match, with the following exceptions:
\begin{enumerate}
	\item The upper bound involves a doubly logarithmic factor that depends on $1 + \frac{\sigma}{\epsilon}$. This is a consequence of the law of the iterated logarithm, which Algorithm~\ref{AlgorithmNoise} uses to maintain uniform correctness of its confidence intervals over time. 
	\item Theorem~\ref{thm:upperbound_samples_noisy} is given in terms of $\Lambda(f,\epsilon/6)$, whereas our lower bound is stated in terms of $\Lambda(f,2\epsilon)$.  
	The two quantities can be related by the following proposition, proved in Section~\ref{sec:LambdaScalePropProof}.
	\begin{proposition}\label{LambdaScaleProp} For any $0 < c \le 1$, $\epsilon > 0 $ and any convex $f$, $\omega(f,x,\epsilon) \geq \omega(f,x,c \epsilon) \geq c \omega(f,x,\epsilon)$ for all $x \in [\tleft(f,\epsilon),1-\tright(f,\epsilon)]$.
	Moreover,
	\iftoggle{annals}
	{
	\begin{multline*}
	 \Lambda(f,\epsilon) + \log\frac{\tleft(f,\epsilon)\tright(f,\epsilon)}{\tleft(f,c\epsilon)\tright(f,c\epsilon)} ~\le~ \Lambda(f,c\epsilon)  \\
	 \le~\frac{1}{c}\left\{\Lambda(f,\epsilon) + \log\frac{\tleft(f,\epsilon)\tright(f,\epsilon)}{\tleft(f,c\epsilon)\tright(f,c\epsilon)} \right\}.
	\end{multline*}
	}
	{
	\begin{multline*}
	 \Lambda(f,\epsilon) + \log\frac{\tleft(f,\epsilon)\tright(f,\epsilon)}{\tleft(f,c\epsilon)\tright(f,c\epsilon)} ~\le~ \Lambda(f,c\epsilon)  ~\le~\frac{1}{c}\left\{\Lambda(f,\epsilon) + \log\frac{\tleft(f,\epsilon)\tright(f,\epsilon)}{\tleft(f,c\epsilon)\tright(f,c\epsilon)} \right\}.
	\end{multline*}
	}
	\end{proposition}\vspace{-2pt}
	Hence, ignoring the contributions of the endpoints $\tleft$ and $\tright$, rescaling $\epsilon$ by a multiplicative constant $c$ changes $\Lambda(f,\epsilon)$ by at most $c$.
		\item Lastly, the upper and lower bounds differs in that $\underline{N}(f,\epsilon)$ requires dividing through by $\log(\omega_{\max}/\omega_{\min})$. We conjecture that the lower bound more accurately reflects the true sample complexity; see Remark~\ref{rem:gap_upper_lower}.
	\end{enumerate}

\subsection{Sample Complexity for Passive Designs}

In this section, we show that the sample complexity for estimating a convex function $f$ with an approximately uniform passive design up to error $\epsilon$ is governed by the parameter $\Lammax(f)$. 
\begin{theorem}\label{thm:passive_lb} Consider a (possibly randomized) passive design $\{x_i\}_{i=1}^n$, which is uniform in the sense that, for some $\tau > 1$, and any interval $I = [a,b]$ with $b - a \le 1/n$, one has that $\Exp[|\{x_i : x_i \in [a,b]\}|] \le \tau$. Then, for a universal constant $c$, any $\delta \in (0,1/3)$ and all $f \in \class$ such that $\left(1+\frac{\sigma^2}{\epsilon^2}\right)\Lammax(f,2\epsilon)~ \ge cn\log(1/\delta)/\tau$, there exists an alternative $\widetilde{f} \in \class$ such that
\begin{eqnarray*}
\sup_{g \in \{f,\widetilde{f}\}}\Pr_{\Alg,g}[\|\widehat{f} - g\|_{\infty} \ge \epsilon] \ge \delta.
\end{eqnarray*} 
\end{theorem}
The proof for the above theorem is as follows. Let 
\begin{align*}
x_* := \arg\inf\left\{\omega(f,x,2\epsilon) :x \in [\tleft(f,2\epsilon),1 - \tright(f,2\epsilon)]  \right\}.
\end{align*} 
which intuitively corresponds to the point with the highest local curvature. Further, let $I_* := [x_*-\omega(f,x_*,2\epsilon),x_*+\omega(f,x_*,2\epsilon)]$, so that $1/|I_*| \gtrsim \Lammax(f,2\epsilon)$. If we consider the alternative function
\begin{eqnarray}
\widetilde{f}(x) := f(x) + \I(x \in I^*) \cdot \left(\Sec[f,I^*](x) - f(x)\right),
\end{eqnarray}
then by construction, $\widetilde{f}$ and $f$ differ only on $\Int(I^*)$ and $\|\widetilde{f}- f\|_{\infty} \ge 2\epsilon$.
So if $\Alg$ can estimate $f$ up to $L_\infty$-norm error $< \epsilon$, then $\Alg$ can distinguish between $f$ and $\widetilde{f}$. 
Consequently, standard information-theoretic arguments (Section~\ref{sec:proof_of_lowerbound_noisy}) imply that any sampling algorithm must collect  $\gtrsim (1 + \frac{\sigma^2}{\epsilon^2})\log(1/\delta)$ samples within $\Int(I_*)$. 
Theorem~\ref{thm:passive_lb} then follows by the uniformity of the sampling procedure. 
In the case where the design is passive but not uniform, it is possible that the design performs well on particular functions $f  \in \class$. In Remark~\ref{rem:non_unif_subopt}, we show that nevertheless, if the design is not uniform, it will underperform on a `translation' of $f$.

\begin{remark} \label{rem:piecewiselinear} \textbf{(Piecewise linear)}
If $f$ is Lipschitz and piecewise linear with a constant number of pieces, then from Section~\ref{sec:examples} we have $\Lammax(f,\epsilon) \approx \epsilon^{-1}$ whereas $\Lamavg(f,\epsilon) \approx \log(1/\epsilon)$.
Theorem~\ref{thm:passive_lb} implies that any $(\epsilon,\delta)$-correct \emph{passive sampling} procedure requires $\epsilon^{-3} \log(/\delta)$ measurements whereas Theorem~\ref{thm:upperbound_samples_noisy} says that our \emph{active sampling} procedure takes just $\epsilon^{-2} \log(1/\epsilon) \log(\log(\epsilon^{-1})/\delta)$.  
Thus, after $n$ total samples the $L_\infty$ of passive sampling decays no faster than $(\tfrac{1}{n})^{1/3}$ whereas active sampling decays like $(\tfrac{\log(n)\log\log(n)}{n})^{1/2}$.
\end{remark}

\section{Recursive Secant Approximation}\label{sec:algorithm}

We now introduce the recursive secant approximation algorithm for learning a convex function with noise. We begin by sampling each endpoint $\{0,1\}$ once. Subsequently, let $t = 3,4,\dots$ denote the number of samples taken, and let $\Tcal_t$ denote a binary tree of intervals contained in $[0,1]$, where the children of an interval $I$ are given by $[\xl,\xm]$ and $[\xm,\xr]$. We let $\Lcal(\Tcal_t)$ denote the set of leaves of $\Tcal_t$. By construction, $\calT_t$ immediately satisfies the following properties stipulated in Lemma~\ref{tree:lem}:
\begin{lemma}\label{tree:lem} For any $t \ge 1$, we have $|I \cap I'| = 0$ for any $I \neq I' \in \Lcal(\Tcal_t)$; $\bigcup_{I \in \Lcal(\Tcal_t)} I = [0,1]$; and $\bigcup_{I \in \Lcal(\Tcal_t)}\{\xm,\xl,\xr\} = \bigcup_{I \in \Tcal_t}\{\xm,\xl,\xr\} $. 
\end{lemma}
At each round $t$, we maintain three estimates of $f$.
First, an estimator $\fpoint$ of $f$ defined only at the points $\bigcup_{I \in \Lcal(\Tcal_t)}\{\xm,\xl,\xr\}$. 
Second, a secant-approximation estimator $\fsec$ which extends the domain of $\fpoint$ to all of $[0,1]$ via:
\begin{align}
\forall I \in \Lcal(\Tcal), x \in I, \quad \fsec(x) := \Sec[\fpoint,I](x)~.
\end{align}
Note that $\fsec$ is well defined, since by Lemma~\ref{tree:lem}, for all $x \in [0,1]$, (a) there exists an $I \in \Lcal(\Tcal_t)$ such that $x \in I$ and (b) if $x \in I_1 \cap I_2$ for $I_1,I_2 \in \Lcal(\Tcal_t)$, then $x$ is a common endpoint of $I_1$ and $I_2$, and thus the secant approximations coincide at $x$ so that  $\fsec(x) = \Sec[\fpoint,I_1](x) =\Sec[\fpoint,I_2](x)$. Lastly, since $\fsec$ is not guaranteed to be convex when measurements are noisy, we define an estimator $\fhat$ via an $L_{\infty}$ projection onto $\class$, 
\begin{align}
\fhat \in \arg\inf_{f \in \class}\|f-\fsec\|_{\infty}~.
\end{align}
By definition $\|\fsec - \fhat\|_{\infty} \le \|\fsec - \fst\|_{\infty}$ so that $\|\fhat - \fst\|_{\infty} \le 2  \|\fsec - \fst\|_{\infty}$ by the triangle inequality. 
When not clear from context, we employ the use of a subscript $t$ on $\fpoint_t,\fsec_t,\fhat_t$ to denote these functions once $t$ samples have been taken.  

\subsection{Recursive Secant Approximation without Noise}\label{sec:noiseless_alg}
 \begin{algorithm}[t]
\textbf{Initialize} Tree of intervals $\Tcal_0 = \{[0,1]\}$, estimate $\fpoint(x) = -\infty$ for $x \in [0,1]$ \\
\textbf{For} $x \in \{0,1\}$, $\fpoint(x) \leftarrow f(x)$\\

\textbf{For} samples $t=3,4,\dots$ \\
	\Indp $I^* \leftarrow \arg\max_{I \in \Lcal(\Tcal_t)} \Delta(\fpoint, I)$ (break ties arbitrarily)\\
	\textbf{If} $\fpoint(x_{m(I^*)}) = -\infty$, observe $\fpoint(x_{m(I^*)}) \leftarrow f(x_{m(I^*)})$, $\Tcal_{t+1} \leftarrow \Tcal_t$  \\
	\textbf{Else} Insert $I_1 := [x_{l(I^*)},x_{m(I^*)}]$ and $I_2 [x_{m(I^*)},x_{r(I^*)}]$ into $\Tcal_{t+1}$ as children of $I^*$; observe $\fpoint(x_{m(I_1)}) \leftarrow f(x_{m(I_1)})$.
\caption{Noiseless Recursive Secant Approximation}\label{AlgorithmNoiseless}
\end{algorithm}

To build intuition for Algorithm~\ref{AlgorithmNoise}, we consider the following noiseless variant of our main algorithm, Algorithm~\ref{AlgorithmNoiseless}, where the oracle returns noiseless queries $F(x) = f(x)$. In this case, $\fpoint$ is set to be equal to $f(x)$ at each point $x$ that is queried, and a placeholder value of $-\infty$ elsewhere. 
The algorithm maintains the invariant that, for all $I \in \Lcal(\Tcal_t)$, $\xl$ and $\xr$ have been measured and recorded in $\fpoint$. Moreover, since the queries are noiseless, the secant approximation $\fsec$ is convex and no projection is required.

At each round, Algorithm~\ref{AlgorithmNoiseless} queries the interval $I \in \Lcal(\Tcal_t)$ for which the secant bias $\Delta(\fpoint,I)$ is largest; note that if there is an interval $I$ for which $\xm$ has not been sampled, then $\fpoint(\xm) = -\infty$ and $\Delta(\fpoint,I) = \infty$, and $\xm$ will be queried, with ties broken arbitrarily. In preparation for the analysis of the noise-tolerant algorithm, we shall analyze the stopping time:
\begin{eqnarray}
\stoptnn(\epsilon) := \inf\{t \ge 0: \max_{I \in \calL(\Tcal_t)}\Delta(\fpoint_t,I) \le \epsilon \}~.
\end{eqnarray}
We shall prove the following proposition:
\begin{proposition}\label{prop:noiseless_sample} For all $t \ge \stoptnn(\epsilon)$, $\|\fsec_t - f\|_{\infty} \le 2\epsilon$. Moreover, for any $\alpha \in (0,1)$ we have $\stoptnn(\epsilon) \le 9 + \frac{2(1+\alpha)}{\alpha}\Lamavg(f,(1-\alpha) \epsilon)~$
\end{proposition}
\begin{proof}
Since $\fpoint(x) = f(x)$ for all queried points $x$, we have that, for $t = \stoptnn(\epsilon)$, $\|\fsec_t - f\|_{\infty} \le \max_{I \in \Lcal(\Tcal_t)}2\Delta(\fpoint_t,I) \le 2\epsilon$. Moreover, since for any $t' > t$, $\fsec_{t'}$ is constructed using secant approximations on a refinement of the intervals $\calL(\calT_t)$, $\|\fsec_{t'} - \fst\|_{\infty} \le \|\fsec_{t} - \fst\|_{\infty}$ (see Lemma~\ref{secant_approx_grows}.)

It remains to bound $\stoptnn(\epsilon)$. Let $\calX_t$ denote the set of points sampled at the start of round $t$; in the noiseless setting,  $t = |\calX_t|$, but bounding $|\calX_t|$ will be of broader interest for the noise-tolerant algorithm.
Since $\calX_{\stoptnn(\epsilon)}$ are the endpoints of the intervals $I \in \Lcal(\Tcal_{\stoptnn(\epsilon)})$, which are adjacent, we have $|\calX_{\stoptnn(\epsilon)}| \le 2|\Lcal(\Tcal_{\stoptnn(\epsilon)})| + 1$. 
Moreover, if $\parents(\Lcal(\Tcal_{\stoptnn(\epsilon)}))$ denotes the parent-intervals of $\Lcal(\Tcal_{\stoptnn(\epsilon)})$, we have $|\Lcal(\Tcal_{\stoptnn(\epsilon)})| \le 2|\parents(\Lcal(\Tcal_{\stoptnn(\epsilon)}))|$. 
Thus, to bound $\stoptnn(\epsilon)$, it suffices to bound $|\parents(\Lcal(\Tcal_{\stoptnn(\epsilon)}))|$. We adopt the shorthand $\calI' := \parents(\Lcal(\Tcal_{\stoptnn(\epsilon)}))$. 

We now make a key observation about $\calI'$,  which will allow us to relate $|\calI'|$ to $\Lamavg$:  for every $I \in \calI'$, we have $\Delta(f,I) \ge \epsilon$; if not, then at the round $s < \stoptnn(\epsilon)$ at which $I$ is bisected, we have $\max_{I' \in \Lcal(\Tcal_s)}\Delta(f,I') = \Delta(f,I) < \epsilon$, which implies that $s \ge \stoptnn(\epsilon)$, a contradiction. 
The following lemma, proved in Section~\ref{sec:mid_interval_modulus}, shows that the inequality $\Delta(f,I) \ge \epsilon$ implies that the average modulus on each $I$ cannot be too small. 
\begin{lemma}\label{MidIntervalModulus} Let $[a,b] \subset [0,1]$, $\epsilon>0$, and suppose that $\Delta(f, [a,b]) \ge \epsilon$. Then for any $\alpha \in (0,1)$, 
$\int_{a}^b \omega(f,x,(1 - \alpha)\epsilon)^{-1} dx\ge \frac{2\alpha}{1+\alpha}$.
\end{lemma}
As a consequence, for any $\alpha \ge 0$ and $I$ such that $\Delta(\fpoint_t,I) \ge \epsilon$, we have $\int_{I}^{}\omega(f,(1-\alpha)\epsilon,x)^{-1}dx \ge \frac{2\alpha}{1+\alpha}$. To relate to the integral $\Lamavg$, we observe that the intervals $I \in \calI'$ are disjoint except at their endpoints, which yields
\begin{align*}
\Lamavg(f,(1-\alpha) \epsilon) &= \int_{\tleft(f,\epsilon)}^{1-\tright(f,\epsilon)}\omega(f,(1- \alpha)\epsilon,x)^{-1}dx\\
&\ge \sum_{I \in \calI' : I \subset [\tleft(f,\epsilon),1-\tright(f,\epsilon)]}\int_{I}\omega(f,(1-\alpha)\epsilon,x)^{-1}dx\\
&\overset{(i)}{\ge} \frac{2\alpha}{1+\alpha}  |\{ I \in \calI' : I \subset [\tleft(f,\epsilon),1-\tright(f,\epsilon)]\}| \\
&\overset{(ii)}{\ge} \frac{2\alpha}{1+\alpha} (|\calI'| - 2)~. 
\end{align*}
Here, $(i)$ follows from Lemma~\ref{MidIntervalModulus}, and $(ii)$ is from the following argument: because the leftmost interval $I_{\mathrm{left}}$ has $\Delta(f,I_{\mathrm{left}}) \ge \epsilon$ and contains $0$, $[0,\tleft(f,\epsilon)] \subset I_{\mathrm{left}}$; 
by the same token, $[1-\tright(f,\epsilon),1] \subset I_{\mathrm{right}}$, and so all remaining $|\calI'|-2$ intervals are contained in $[\tleft(f,\epsilon),1-\tright(f,\epsilon)]$. 

In summary, we find that for any $\alpha \in (0,1)$,
\begin{align*}
|\calX_{\stoptnn(\epsilon)}| &\le 2|\calL(\Tcal_{\stoptnn(\epsilon)})| + 1 ~\le 4|\calI'| + 1\\
&\le 9 + \frac{2(1+\alpha)}{\alpha}\Lamavg(f,(1-\alpha)\epsilon) \, .
\end{align*}
\end{proof}
We remark that our bound on $|\calX_{t}|$ only used the fact that at time $t \le \stoptnn(\epsilon)$ we had $\Delta(f,I) \ge \epsilon$ for each $I \in \parents(\Lcal(\Tcal_t))$. This observation will be essential in generalizing to the setting with a noise oracle.

\subsection{Recursive Secant Approximation with Noise}\label{sec:noisy_sample_complexity}


\begin{algorithm}[ht]
\textbf{Input} Bias-variance tradeoff $\beta$, confidence parameter $\delta$, oracle $F$, confidence functions $B_t(\cdot,\cdot)$ and $\phi(\cdot,\cdot)$, mutable maps $\fpoint(\cdot)$,  $\deltil(\cdot)$, $\eps_0 = \infty$\\
\textbf{Initialize} Tree of intervals $\Tcal = \{[0,1]\}$, $\fpoint(x)=-\infty$ and $N_t(x)=0$ for $x \in [0,1]$, and $\deltil(x)=\delta/6$ for $x \in \{0,1/2,1\}$ \\
\textbf{Sample} Sample at $y_1 \sim F(0)$, and $y_2 \sim F(1)$, and let $\fpoint_2(0) = y_1$ and $\fpoint_2(1) = y_2$, update $N_2(0) \leftarrow 1$, $N_2(1) \leftarrow 1$. \\
\textbf{For} round $t=2,3,\dots$ \\
	
	 \Indp   $\epsilon_t = \epsilon_{t-1} \wedge \max_{I \in \Lcal(\Tcal)} 4( B_t(I,\deltil) + 4\max \{0, \Delta(\fpoint_t, I)\}) $ \label{line:epst}\\
	 \textbf{If} $\epsilon_t < \epsilon_{t-1}$,  $\widehat{f}_{t} \leftarrow \arg\inf_{f \in \class}\|f-\fsec_{t}\|_{\infty}$\textbf{; else}  $\widehat{f}_{t} \leftarrow \widehat{f}_{t-1}$. \\

	$I^*_t \leftarrow \arg\max_{I \in \Lcal(\Tcal)} B_t(I,\deltil) + \max \{0, \Delta(\fpoint_t, I)\}$ \\
	\textbf{Sample} $y_t \sim F(x_t)$, where\label{ln:sample}
	\begin{align*}
	x_t \in \arg\max\{\phi(N_t(x),\deltil(x)) ~|~ x \in \{\xl,\xr,\xm\}\} 
	\end{align*}
	%
	\\
	\textbf{Update} $N_{t+1}(x_t) \leftarrow N_t(x_t) + 1$,  $\fpoint_{t+1}(x_t) \leftarrow y_t \cdot \frac{1}{N_{t+1}(x_t)} + \fpoint_t(x_t)\cdot \frac{N_{t+1}(x_t) -1 }{N_{t+1}(x_t)} $\\
	 \textbf{If} $ (1+\beta)B_{t}(I^*_t,\deltil)  < \Delta(\fpoint_t, I^*_t)$ \label{ln:while_begin}\\
	\Indp Bisect $I^*_t$ into two even intervals $I_1$ and $I_2$ \\
	\textbf{For} $j=1,2$, \textbf{append} $I_j$ to $\Tcal_{t+1}$ as a child of $I^*_t$, $\deltil( x_{m(I_j)} ) \leftarrow \delta/2|\Lcal(\Tcal_t)|^2$ \label{ln:delta_update}, $\fpoint( x_{m(I_j)} ) \leftarrow -\infty$, $N_t( x_{m(I_j)} ) \leftarrow 0$\\
\caption{Active Learning Algorithm for Convex Regression}\label{AlgorithmNoise}
\end{algorithm}


We now describe how to generalize Algorithm~\ref{AlgorithmNoiseless} to allow for noisy observations. Fix some time $t$ and let $\{(x_s,y_s)\}_{s=1}^t$ be the collection of noisy function evaluation pairs.
Recall that $y_s = f(x_s) + \noise_s$ where $\noise_s$ is independent, mean-zero $\sigma^2$-sub-Gaussian distributed noise, i.e. $\E[\exp(\lambda \noise_s)] \leq \exp(\lambda^2\sigma^2/2)$.
In the algorithm, $N_t(x) = \sum_{s=1}^t \1\{ x_s = x \}$ will denote the number of times the point $x \in [0,1]$ has been sampled so that $\fpoint(x) = \frac{1}{N_t(x)} \sum_{s=1}^t  \1\{ x_s = x \} \ y_s$ if $N_t(x) \ge 1$, and $-\infty$ otherwise. Lastly, we let $\phi(t,\delta)$ denote an anytime confidence interval such that 
\begin{align*}\P\left( \bigcup_{t=1}^\infty \{ |\frac{1}{t} \sum_{s=1}^t w_s| \geq \phi(t,\delta) \} \right) \leq \delta~.
\end{align*}
For example, $\phi(t,\delta) = \sqrt{ 16 \sigma^2 \log(\log_2(2t)/\delta)/t}$ suffices but we recommend using \citet[Theorem 8]{kaufmann2016complexity}. In general $\phi(\cdot,\cdot)$ can be chosen to be monotically decreasing in the $t$-argument, and increasing in the $\delta$-argument.
In addition to $N_t$ and $\fpoint$, we maintain a function $\deltil: [0,1] \to \R_{> 0}$ such that 
\begin{align*}\Pr\left[\forall t \ge 1, x \in \text{supp}(N_t): |\frac{1}{N_t(x)} \sum_{s=1}^t \1\{ x_s = x \}w_s| \leq \phi(N_t(x),\deltil(x))\right] \ge 1- \delta~.
\end{align*}
We shall let $\Egood$ denote the event inside the probability operator in the above display. Finally, define confidence bounds 
\begin{align*}
B_t(I,\deltil) :=& \phi(N_t(\xm),\deltil(\xm)) \\
&+ \frac{1}{2} \max\{ \phi(N_t(\xl),\deltil(\xl)), \phi(N_t(\xr),\deltil(\xr))\}~.
\end{align*}
 Crucially, our confidence bounds ensure the following sandwich relation, proved in Section~\ref{sec:proof:sandwhich}:
\begin{lemma}\label{lem:sandwhich} On $\Egood$, the following holds for all $t \ge 1$ and $I \in \Lcal(\Tcal_t)$: $\sup_{x \in I} |\Sec[\fpoint_t,I](x) - f(x) |\leq 2\left(\max\{0,\Delta(\fpoint_t, I)\} + B_t(I,\deltil_t)\right)$. 
\end{lemma}
As a consequence, we find that
\begin{align}\label{eq:f_secerror}
\|\fsec_t - f(x)\|_{\infty} \le 2\max_{I}\left(\max\{0,\Delta(\fpoint_t, I)\} + B_t(I,\deltil_t)\right) = \epsilon_t /2~.
\end{align}

For a fixed parameter $\beta >0$, Algorithm~\ref{AlgorithmNoise} maintains the condition $(1+\beta)B_t(I,\deltil) \geq \Delta(\fpoint_t, I)$ using the while loop of Line~\ref{ln:while_begin}.
This is to ensure that the stochastic variance always dominates the bias of the approximation.
The parameter $\beta>0$ appears to have little effect on performance as long as it is smaller than $1$; we recommend setting $\beta =1/2$.
The definition of $I^*$ in the algorithm is motivated by the sandwich relationship (Lemma~\ref{lem:sandwhich}) noted above.
And in each case, $x_t$ is chosen in order to minimize the maximum confidence bound relevant to the interval $I^*$.
The values of $\deltil(x_{m(I_j)})$ satisfy $\sum_{x : T(x) > 0} \deltil(x) \leq \delta$ since $3 \cdot \frac{1}{6} + \sum_{k=2}^\infty \frac{1}{2 k^2} \leq 1$.

\subsection{Proof of Upper Bound, Theorem~\ref{thm:upperbound_samples_noisy}}\label{sec:upperbound_samples_noisy_proof} 
Recall the definition set $\Xcal_t := \bigcup_{I \in \Lcal(\Tcal_t)}\{\xm,\xr,\xl\}$, and we shall assume that $\Egood$ holds. Fix an $\epsilon > 0$, and let $\epsilon_t$ be as in Algorithm~\ref{AlgorithmNoise} Line~\ref{line:epst}, and define the stopping time 
\begin{align*}
\stopt(\epsilon) :=& \inf\{ t \ge 1: \epsilon_t \leq \epsilon\} \\
=& \inf \{t \ge 1: \max_{I \in \Lcal(\Tcal)} B_t(I,\deltil) + \max \{0, \Delta(\fpoint_t, I)\} \leq \epsilon/4 \}.
\end{align*}
The correctness guarantee is a direct consequence of~\eqref{eq:f_secerror} since on $\Egood$, $\|\fhat_t - f\|_{\infty} \leq 2\|\fsec_t - f\|_{\infty} \le 2  \cdot \epsilon_t/2 \leq \epsilon$.
Because $\widehat{f}_t$ is only updated using a decreasing sequence of values of $\epsilon_t$, the guarantee immediately holds for all $t' \ge \stopt(\epsilon)$. 
In order to upper bound $\stopt(\epsilon)$, we have the identity
\begin{eqnarray}
\stopt(\epsilon) = 1 + \sum_{x \in \calX_{\stopt(\epsilon)}} N_{\stopt(\epsilon) - 1}(x)~.
\end{eqnarray}
Thus, a crucial part of bounding $\stopt(\epsilon)$ is showing that we do not \emph{oversample} $x \in \calX_t$; this is accomplished by relating the stopping condition to the sampling rule.
\begin{lemma}\label{lem:N_tau} $\forall x \in \calX_{\stopt(\epsilon)}$, $N_{\tau(\epsilon) - 1}(x) \le 1 \vee \max_{s \ge 1} \{\phi(s,\deltil(x)) \ge \frac{\epsilon}{6(2+\beta)}\}$.
\end{lemma}
As a consequence, 
\begin{align*}
\tau(\epsilon) &\le 1 + |\calX_{\tau(\epsilon)}| (1 + \max_{x \in \calX_{\tau(\epsilon)}} \cdot \max_{s\ge 1}\left\{\phi(s,\widetilde{\delta}(x)) \ge \frac{\epsilon}{6(2+\beta)}\right\}) \\
&\le 1 + |\calX_{\tau(\epsilon)}| (1 +  \max_{s\ge 1}\left\{\phi(s,\frac{1}{2|\calX_{\tau(\epsilon)}|^2}) \ge \frac{\epsilon}{6(2+\beta)}\right\})~,
\end{align*}
where the second line uses the fact that $\phi(\cdot,\cdot)$ is monotone in its second argument, and $\max_{x \in \calX_t} \deltil(x) = 1/2|\calX_t|^2$. We can upper bound the inversion of $\phi(\cdot,\cdot)$ to yield (see e.g.~\cite{kaufmann2016complexity})
\begin{align*}
\tau(\epsilon) &\lesssim \sigma^{2} (2+\beta)^2 \epsilon^{-2} |\calX_{\tau(\epsilon,\beta)}| \log( |\calX_{\stopt(\epsilon,\beta)}| \log((2+\beta)^2 \epsilon^{-2})/\delta)~.
\end{align*}
To wrap up, it suffices to prove that for some $\alpha \in (0,1)$
\begin{eqnarray}
|\calX_{\stopt(\epsilon)}| \le 9 + \frac{2(1+\alpha)}{\alpha} \Lamavg(f, (1-\alpha)\tfrac{\beta}{2(2+\beta)} \epsilon).
\end{eqnarray}
Recalling the argument from Section~\ref{sec:noiseless_alg}, it suffices only to verify that, if $I \in \Lcal(\Tcal_t)$ for $t = \tau(\epsilon)$, then the secant approximation error of its parent $I'$ is lower bounded by $\Delta(f,I')  \ge  \frac{\beta \epsilon}{2 (2+\beta)}$. We prove this as follows:
fix some $I \in \Lcal(\Tcal_t)$ for $t = \tau(\epsilon)$.
If $I'$ is the parent of $I$ then there exists some previous time $s < t$ such that  
\begin{align*}
(1+\beta)B_s(I', \deltil) &< \Delta(\fpoint_s , I') \overset{\Egood}{\le} \Delta(f, I') + B_s(I', \deltil)~,
\end{align*}
that is, $\Delta(f,I') \ge \beta B_s(I', \deltil)$. On the other hand, to split on $s < \tau(\epsilon)$ we must also have that
\begin{align*}
\epsilon/4 &< B_s(I',\deltil) + \max \{0, \Delta(\fpoint_s, I')\} \\
&\leq 2 B_s(I',\deltil) + \Delta(f, I'),
\end{align*}
Together, these two displays imply $\Delta(f,I') \ge \beta B_s(I', \deltil) \geq \beta(\epsilon/4 - \Delta(f,I'))/2$. 
Rearranging, we find $\Delta(f,I') \geq \frac{\beta}{2(2+\beta)}\epsilon$ which proves what we set out to verify.

\subsection{Proof of Lemma~\ref{MidIntervalModulus}\label{sec:mid_interval_modulus}}
Fix $\alpha \in (0,1)$. This proof relies on the following upper-continuity property of $\omega$, whose proof is deferred to Section~\ref{sec:ModulusUB}:
\begin{lemma}\label{ModulusUB} Let $[x-t,x+t] \subset [0,1]$, and suppose that $\Delta(f,x,t) \ge \epsilon$. Then, $\omega(f,x+\tau,\epsilon(1 - \frac{|\tau|}{t})) \le t + |\tau|$.
\end{lemma}
For $u \in [x, x+\alpha t]$, we have $1 - \alpha \leq 1-\tfrac{u-x}{t}$. Since $\omega(f,u,\cdot)$ is monotone in its third argument,
\begin{align*}
\omega(f,u,\epsilon(1 - \alpha)) \leq \omega(f,u,\epsilon(1-\tfrac{u-x}{t}))
\end{align*}
and, making a substitution $\tau = u - x$, 
\begin{align*}
\int_{x}^{x+t} \omega(f,u,\epsilon(1 - \alpha))^{-1} du &\geq \int_{x}^{x+\alpha t} \omega(f,u,\epsilon(1 - \tfrac{u-x}{t}))^{-1} du \\
&= \int_{0}^{\alpha t} \omega(f,x+\tau,\epsilon(1 - \tfrac{\tau}{t}))^{-1} d\tau \\
&\overset{(\text{Lemma}~\ref{ModulusUB})}{\geq}\int_{0}^{\alpha t} (t+\tau)^{-1} d\tau \\
&= \alpha t \cdot (t+\alpha t)^{-1}~ = \alpha/(1+\alpha)~,
\end{align*}
which proves one side of the integral. A similar argument holds for $u \in [x-\alpha t,x]$ since $1-\alpha \leq 1- \frac{|u-x|}{t}$. 

\subsection{Proof of Lemma~\ref{lem:sandwhich}\label{sec:proof:sandwhich}}
	Define $\widetilde{r}(x) = \fpoint(x) - f(x) \quad$ for all $x \in \supp(\fpoint)$.
	First note that
	\begin{align*}
	\Sec[\fpoint,I](x) - f(x) &= \Sec[\fpoint,I](x) - \Sec[f,I](x) + \Sec[f,I](x) - f(x) \\
	&\leq \max\{ \widetilde{r}(\xl), \widetilde{r}(\xr) \} + 2\Delta(f,I)
	\end{align*}
	using the fact that the secant approximations are affine on $I$ and $\Sec[f,I](x) - f(x) \leq 2\Delta(f,I)$ by Lemma~\ref{lem:mainLemmaSimple}. Adding and subtracting $2\Delta(\fpoint, I)$,
	\begin{align*}
	\Sec[&\fpoint,I](x) - f(x) \\
	&\leq \max\{ \widetilde{r}(\xl), \widetilde{r}(\xr)\}  + 2( \Delta(f, I) - \Delta(\fpoint, I)) + 2 \Delta(\fpoint, I) \\
	&= \max\{ \widetilde{r}(\xl), \widetilde{r}(\xr)\} - 2(\tfrac{\widetilde{r}(\xl) + \widetilde{r}(\xr)}{2} - \widetilde{r}(\xm) ) + 2 \Delta(\fpoint, I)\\
	&= 2 \Delta(\fpoint, I) -\min\{ \widetilde{r}(\xl), \widetilde{r}(\xr) \} + 2 \widetilde{r}(\xm))~,
	\end{align*} 
	whence we conclude $\tfrac{1}{2}(\Sec[\fpoint,I](x) - f(x)) \le  \Delta(\fpoint, I) + {B}(I, \widetilde{\delta})$ on $\Egood$. For the lower bound, we see
	\begin{align*}
	\Sec[\fpoint,I](x) - f(x) &= \Sec[\fpoint,I](x) - \Sec[f,I](x) + \Sec[f,I](x) - f(x)  \\
	&\geq \min\{ \widetilde{r}(\xl), \widetilde{r}(\xr) \} ~.
	\end{align*}
	so that $-\tfrac{\Sec[\fpoint,I](x) - f(x)}{2} \leq B(I, \widetilde{\delta})$ on $\Egood$.
	Thus,
	\begin{align*}
	-B(I, \widetilde{\delta}) \leq \tfrac{\Sec[\fpoint,I](x) - f(x)}{2} \leq \Delta(\fpoint, I) + {B}(I, \widetilde{\delta})~.
	\end{align*}
	\begin{remark}\label{remark:splitting}
	In the proof of Lemma~\ref{lem:sandwhich} we lower bound $\min\{ \widetilde{r}(\xl), \widetilde{r}(\xr) \} $ by $-B(I,\widetilde{\delta})$ which is quite loose since this quantity is also lower bounded by $-\max\{ \phi(T(\xl),\deltil(\xl)), \phi(N_t(\xr),\deltil(\xr))\}$ and can be at least a factor of two smaller. 
	Nevertheless, using matching upper and lower bounds for $\Sec[\fpoint,I](x) - f(x)$ substantially simplifies clutter in the algorithm.
	It is straightforward to modify the algorithm to use these non-matching upper and lower bounds for superior empirical performance, and, indeed, our experiments implement this modification.
	\end{remark}


\subsection{Proof of Lemma~\ref{lem:N_tau}} 
Fix an $x^* \in \calX_{\stopt(\epsilon)}$, and let $s < \stopt(\epsilon)$ be the last round at which $x^*$ was sampled; note then that $x^* \in I^*_s$. 
It suffices to bound $N_s(x^*)$. 
%
%
If $I^*_s$ is a new, just-bisected interval then we must have that $x^* = x_{m(I^*_s)}$ by the sampling rule ($\phi(0,\cdot) = \infty$) so that $x^*$ was sampled only a single time.

Otherwise, $x^*$ has been sampled more than once and $\max \{0, \Delta(\fpoint_s, I^*_s)\} \leq (1+\beta)B_s(I^*_s,\deltil)$. 
This means that for $I^*_s$ one has that 
\begin{align*}
B_s(I^*_s,\deltil) + &\max \{0, \Delta(\fpoint_s, I^*_s)\} \leq (2+\beta)B_s(I^*_s,\deltil) \\
&\leq  \frac{3(2+\beta)}{2}  \max_{x \in \{ x_{l(I^*_s),x_m(I^*_s),x_r(I^*_s)}\}} \phi(N_t(x),\deltil(x)) \\
&= \frac{3(2+\beta)}{2} \phi(N_t(x^*),\deltil(x^*))
\end{align*}
where the last line follows by the sampling rule. 
It suffices for the right-hand side to be less than $\epsilon/4$ to meet the stopping condition.



\section{Proof of Packing and Lower Bounds} 
\subsection{ Proof of Theorem~\ref{thm:packing}\label{sec:packing_proof}}

We construct the packing by choosing a sequence of interval midpoints $m_i$ and interval lengths $t_i$, such that the intervals $I_i := [m_i- t_i,m_i + t_i] = [a_i,b_i]$ overlap only at their endpoints, and such that $\Delta(f,m_i,t_i) = \epsilon$. To do this, we define $t_0 = m_0 = \tleft(f,\epsilon)$. By definition of $\tleft(f,\epsilon)$, we have the equality $\Delta(f,\tleft(f,\epsilon), \tleft(f,\epsilon)) = \epsilon$. Let $b_0 = 0$, and for each $i \ge 1$, we define
\begin{align*}
t_{i} &:= \sup \left\{t \in [0,\frac{1 - b_{i-1}}{2}]:  \Delta(f,b_{i-1} + t, t) \le \epsilon\right\} \\
(a_i,m_i,b_i) &:= (b_{i-1}, b_{i-1}+t_i, b_{i-1} +2t_i)~.
\end{align*}
One can think of $t_i$ as as the equivalent of $\tleft$, but starting at $b_{i-1}$ rather than zero. Note that $\Delta(f,b_{i-1} + t, t)$ is non-decreasing and continuous in $t$ (Lemma~\ref{monotone_lem}), and thus, if there exists a $t \in [0,\frac{1 - b_{i-1}}{2}]$ such that $\Delta(f,b_{i-1} + t, t) \ge \epsilon$, then the supremum in the definition of $t_i$ will be attained for a $t_i$ such that $\Delta(f,b_{i-1} + t, t)= \Delta(f,m_i,t_i) = \epsilon$. Thus,  we will terminate the construction at $i= n$, where $n$ is the first number satisfying $b_n \ge 1 - 2\tright(f,\epsilon)$, or $\Delta(f,b_n + t_{n+1},t_{n+1}) < \epsilon$. Note that $b_n = b_{i-1} + 2t_i \le b_i + 2(\frac{1 - b_{i-1}}{2}) = 1$.
%
Collecting what we have established thus far,
\begin{enumerate}
	\item $\Delta(f,m_i,t_i) = \Delta(f,b_{i-1}+t_i,t_i) =  \epsilon$. By Lemma~\ref{eps_f_u_lem}, it follows that $t_i = \omega(f,m_i,\epsilon)$.
	\item By definition, $a_1 = 2\tleft(f,\epsilon)$. And, by the stopping condition, $a_n \le 1 - 2\tright(f,\epsilon) \le b_n \le 1$
	\item Hence, since $b_i = a_{i+1}$, we have that $\bigcup_{I_i} = [a_1,b_n] \supseteq  [2\tleft(f,\epsilon),1 - 2\tright(\epsilon,f)]$, and that $I_i$ have disjoint interiors.
\end{enumerate}
To conclude, we adopt an argument similar to the proof of Proposition~\ref
{prop:noiseless_sample}. For ease of notation, define $\Ibar(f,\epsilon) := [\tleft(f,\epsilon), 1-\tright(f,\epsilon)]$. We start off by showing that $\int_{m-t}^m\omega(f,u,\epsilon)^{-1}du  = \widetilde{\mathcal{O}}(1)$ for $m \in \Ibar(f,\epsilon)$.
\begin{lemma}\label{lem:LB_integral}
Let $m \in \Ibar(f,\epsilon)$ and $t = \omega(f,m,\epsilon)$, so that $\Delta(f,m,t) = \epsilon$. Then if $\omega_{0} = \inf_{u \in [m-t,m+t]}\omega(f,u,\epsilon)$, one has
\begin{align*}
\int_{m-t}^m\omega(f,u,\epsilon)^{-1}du \vee \int_{m}^{m+t}\omega(f,u,\epsilon)^{-1}du \le 2\left(1 + \log\frac{t}{\omega_{0}}\right )~.
\end{align*}
\end{lemma}
In particular, for all $i \in [n]$, we have the bound
\begin{align*}
\int_{I_i }\omega(f,u,\epsilon)^{-1} du &\le 4\log\left(1 + \frac{\omega(f,m_i,\epsilon)}{\inf_{u \in [m-t,m+t]}\omega(f,u,\epsilon)}\right)\\
 &\le 4\log\left(1 + \frac{\omega_{\max}(f,\epsilon)}{\omega_{\min}(f,\epsilon)} \right)~.
\end{align*}
As a result, we find that
\begin{align*}
\int_{2\tleft(f,\epsilon)}^{1 - 2\tright(f,\epsilon)}\omega(f,u,\epsilon)^{-1}du  &\le \int_{a_1}^{b_n}\omega(f,u,\epsilon)^{-1}du ~\le \sum_{i=1}^n \int_{I_i}\omega(f,u,\epsilon)^{-1}du \\
&\le n\cdot 4\log\left(1 + \frac{\omega_{\max}(f,\epsilon)}{\omega_{\min}(f,\epsilon)} \right)~.
\end{align*}

By the same token, Lemma~\ref{lem:LB_integral} implies
\begin{align*}
\int_{\tleft(f,\epsilon)}^{2\tleft(f,\epsilon)}\omega(f,u,\epsilon)^{-1}du + \int_{1-2\tright(f,\epsilon)}^{1-\tright(f,\epsilon)}\omega(f,u,\epsilon)^{-1}du \le 2 \cdot 2(1 + \log\frac{\omega_{\max}}{\omega_{\min}} )~.
\end{align*}
\noindent Hence,
\begin{align*}
\int_{\tleft(f,\epsilon)}^{1 - \tright(f,\epsilon)}\omega(f,u,\epsilon)^{-1}du &\le 4(n+1)(1 + \log\frac{\omega_{\max}}{\omega_{\min}} ),
\end{align*}
and thus the number of intervals satisfies 
\begin{align*}
n \ge \frac{1}{4(1 + \log\frac{\omega_{\max}}{\omega_{\min}})}\int_{\tleft(f,\epsilon)}^{1 - \tright(f,\epsilon)}\omega(f,u,\epsilon)^{-1}du - 1~.
\end{align*} 
Finally, we remove the last interval $I_n$. Since the right endpoint $a_n$ of $I_n$ satisfies $a_n \le 1 - 2\tright(f,\epsilon)$, the intervals $I_1,\dots,I_{n-1}$ are contained within $[2\tright(f,\epsilon),1-2\tright(f,\epsilon)]$, and we have
\begin{align*}
n - 1 \ge \frac{1}{4(1 + \log\frac{\omega_{\max}}{\omega_{\min}})} \int_{\tleft(f,\epsilon)}^{1 - \tright(f,\epsilon)}\omega(f,u,\epsilon)^{-1}du - 2 = \underN(f,\epsilon) + 1.
\end{align*}
Note then that we may take $n - 1= \Npack(f,\epsilon)$.

\subsubsection{Proof of Lemma~\ref{lem:LB_integral}}
We first need a technical lemma, which we prove in Section~\ref{sec:lem:LB_omega_local}.
\begin{lemma}\label{lem:LB_omega_local} Let $x \in \overline{I}(f,\epsilon)$, and $\tau \in [-1,1]$, such that $u := x + \tau \omega(f,x,\epsilon) \in \overline{I}(f,\epsilon)$. Then,
\begin{eqnarray}
\omega(f,u,\epsilon) \ge \frac{(1-|\tau|) \omega(f,x,\epsilon)}{2}~.
\end{eqnarray}
\end{lemma}
We shall now establish $\int_{m}^{m+t}\omega(f,u,\epsilon)^{-1}du \le 2\left(1 + \log\frac{t}{\omega_{0}}\right )$; the bound on the integral over $[m-t,m]$ is analogous. We can write $u \in [m,m+t]$ as $u = m + \tau t$, where $t=\omega(f,m,\epsilon)$ and $\tau \in [0,1]$. Now, set $\omega_{\min} = \min_{u \in [m,m+1]}(f,u,\epsilon)$. Using Lemma~\ref{lem:LB_omega_local}, we can integrate
\begin{eqnarray*}
\int_{m}^{m+t}\omega(f,u,\epsilon)^{-1}du &=& t \int_{ 0}^{1}\omega(f,m + \tau t,\epsilon)^{-1}d\tau \\
&\overset{(a)}{\le}& t\int_{0}^{1}\min\left\{\frac{2}{(1-\tau)t},\omega_{0}^{-1}\right\}d\tau\\
&=& \int_{0}^{1}\min\left\{2/\tau,t/\omega_{0}\right\}d\tau \\
&=& \int_{ 0}^{2\omega_{0}/t}t/\omega_{0} d\tau' + \int_{2\omega_{0}/t}^1 2/\tau' \ d\tau'\\
&=& 2 + 2 \log(1 \wedge t/(2\omega_{0}))~,
\end{eqnarray*}
where $(a)$ is precisely Lemma~\ref{lem:LB_omega_local}. Lastly, we can bound $ \log(1 \wedge t/(2\omega_{0})) \le \log(t / \omega_{0})$, since $\omega_{0} \le t = \omega(f,m,\epsilon)$. 

\subsection{Proof of Noisy Lower Bound, Theorem~\ref{thm:lowerbound_noisy}}\label{sec:proof_of_lowerbound_noisy} It suffices to prove the theorem with $\underN$ replaced by $\Npack$, since $\Npack(f,\cdot) \ge \Npack(f,\cdot)$; the case where $\Npack(f,2\epsilon)  = 0$ is addressed at the end of the section. Let $\Alg$ be any algorithm satisfying the correctness guarantee~\eqref{eq:correctness} for some $\epsilon > 0$ and $\delta \in (0,1/3)$. For $g \in \class$, let $\Pr_{g}$ denote the law under $g$ and $\Alg$. 
Consider the local alternative class $\mathcal{G}_{\fst,2\epsilon} \subset \class$ and intervals $\calI_{\fst,2\epsilon} := \{I_i\}_{i=1}^{\Npack(\fst,2\epsilon)}$, where $\Npack(\fst,\cdot)$ and $G_{\fst,\cdot}$ are defined Theorem~\ref{thm:packing} and~\eqref{Local_Alternatives_Eq}, respectively.~Let $\stopt_i$ denote the random variable corresponding to the number of times $\Alg$ samples in the interior of $\calI_i$, and observe that since the intervals in $i$ have disjoint interiors, the stopping time $\stopt$ of $\Alg$ satisfies $\sum_{i}\stopt_i \le \stopt$. 

We can reduce to a multiple hypothesis testing problem by recalling that, for $h \ne g \in \calG_{\fst,2\epsilon}$, $\|h -g \|_{\infty} \in [2\epsilon,2\cdot 2\epsilon]$. Hence, for $g \in \calG_{\fst,2\epsilon}$, the events $\calA_g:= \{\|\fhat - g\|_{\infty} < \epsilon\}$ are pairwise disjoint. Further, by~\eqref{eq:correctness}, one has $\Pr_{g}[\calA_g] \ge 1 -\delta,\forall g \in \calG_{2\epsilon,\fst}$. We also recall Birge's inequality:
\begin{lemma}[Birge's Inequality, Theorem 4.21 in~\cite{boucheron2013concentration}] Let $\Pr_0,\Pr_1,\dots,\Pr_n$ denote a family of probability distributions on a space $(\Omega,\mathcal{F})$, and let $\calA_0,\calA_1,\dots,\calA_n$ denote pairwise disjoint events. If $p:= \min_{i} \Pr_i(A_i) \ge 1/(n+1)$, then
\begin{eqnarray}
\frac{1}{n}\sum_{i=1}^n \KL(\Pr_i;\Pr_0) \ge \kl(p, \frac{1 - p}{n}),
\end{eqnarray}
where $\kl(a,b) := a \log \frac{a}{b} + (1-a) \log \frac{1-a}{1-b}$. 
\end{lemma}
To apply Birge's inequality, we first compute $\KL(\Pr_{g},\Pr_{h})$ for any $g,h \in \calG_{2\epsilon,f}$ such that $g(x) = h(x)$ for all $x \in [0,1] \setminus \Int(I_i)$, where $I_i \in \calI_{\fst,2\epsilon}$. Let $\KL(g(x),h(x))$ denote the $\KL$ between $\mathcal{N}(g(x),\sigma^2)$ and $\mathcal{N}(h(x),\sigma^2)$, which is equal to $(g(x)-h(x))^2/2\sigma^2$. Then
\begin{align*}
\KL(\Pr_{g};\Pr_{h}) &= \Exp_{g}[\sum_{s=1}^T \KL(g(X_s),h(X_s))] \\
&\overset{(i)}{=} \Exp_{g}[\sum_{s: X_s \in \Int(I_i)} \KL(g(X_s),h(X_s))]\\
&= \frac{1}{2\sigma^2}\Exp_{g}[\sum_{s: X_s \in \Int(I_i)} (g(X_s)-h(X_s))^2 ]\\
&\le \frac{1}{2\sigma^2}\Exp_{g}[\stopt_i]\cdot \sup_{x \in I_i}(g(x) - h(x))^2\\
&\overset{(ii)}{\le} \Exp_{g}[\stopt_i] \cdot 8\epsilon^2 / \sigma^2,
\end{align*}
where $(i)$ uses the fact that $g$ and $h$ differ only on $I_i$, $(ii)$ uses the fact that, on $I_i$, one of $\{g,h\}$ is equal to $\fst$, one is equal to $\Sec[\fst,I_i](x)$, and thus by Lemma~\ref{lem:mainLemmaSimple}, we have that 
\begin{align*}
\left|\max_{x \in I^{\epsilon}_i} \Sec[\fst,I^{\epsilon}_i](x) - \fst(x) \right|\le 2\Delta(\fst,I_i) =  2\cdot 2\epsilon~.
\end{align*}

\textbf{First part of Theorem~\ref{thm:lowerbound_noisy}:} For each $i \in [\Npack(f,2\epsilon)]$, let $g^{(i)}$ denote the alternative corresponding to the vector $\bin^{(i)}_j := \I(i = j)$ in~\eqref{Local_Alternatives_Eq}. Hence, $\fst$ and $g^{(i)}$ differ only on $I_i$, and thus 
\begin{align*}
\KL(\Pr_{\fst},\Pr_{g^{(i)}}) \le   \Exp_{\fst}[\stopt_i] \cdot 8\epsilon^2 / \sigma^2.
\end{align*}
Birge's inequality with $n = 1$, $\Pr_1 = \Pr_{\fst}$ and $\Pr_0 = \Pr_{g^{(i)}}$, and $\calA_1 = A_{\fst}$ and $\calA_0 = \calA_{g^{(i)}}$  implies
\begin{align*}
\frac{8\epsilon^2}{\sigma^2} \cdot \Exp_{\fst}[\stopt_i] &\ge \Exp_{\fst}[\sum_{s=1}^T \KL(\fst(X_s),g^{(i)}(X_s))] \ge \kl(1-\delta,\delta)~.
\end{align*}
We rearrange to get $\Exp_{\fst}[\stopt_i]  \ge \sigma^2\kl(1-\delta,\delta)/8\epsilon^2$, and sum over $i \in [\Npack(\fst,\epsilon)]$ to obtain $\Exp_{\fst}[\stopt] \gtrsim \frac{\sigma^2}{\epsilon^2}\Npack(\fst,2\epsilon)\cdot \kl(1-\delta,\delta) \gtrsim \frac{\sigma^2}{\epsilon^2}\Npack(\fst,2\epsilon) \log(1/\delta)$, where the last inequality holds for $\delta \in (0,1/3)$.

\textbf{Second part of Theorem~\ref{thm:lowerbound_noisy}: }
Let $n = \Npack(\fst,2\epsilon)$, and recall the functions  $g_{\bin} \in \mathcal{G}_{\fst,2\epsilon}$ defined in Equation~\ref{Local_Alternatives_Eq}, where $\bin \in \{0,1\}^{n}$. It will be convenient to introduce the notation $\bin \oplus i \in \{0,1\}^{n}$ to denote the vector that agrees with $\bin$ except for flipping the $i$-th bit. Since $g_{\bin}$ and $g_{\bin \oplus i}$ differ on $I_i$ and nowhere else, we have $\calA_{g_{\bin}} \cap \calA_{g_{\bin \oplus i}} = \emptyset $. Again, by correctness, $\Pr_{g_\bin}[\calA_{g_{\bin}}] \ge 1 - \delta \ge 1/2 \ge 1/(n+1)$. Hence, applying Birge's inequality with $\Pr_{0} = \Pr_{g_\bin}$, $\Pr_{i} = \Pr_{g_{\bin \oplus i}}$, and the disjoint events $\calA_0 = \calA_{g_{\bin}}$ and $\calA_{i} = \calA_{g_{\bin\oplus i}}$, we have that for any $\bin \in \{0,1\}^{n}$,
\begin{align*}
\kl(1-\delta,\delta/n) \le \frac{1}{n}\sum_{i=1}^n \KL(\Pr_{i};\Pr_0) \le \frac{8\epsilon^2}{\sigma^2} \cdot \frac{1}{n}\sum_{i=1}^n \Exp_{g_{\bin\oplus i}}[\stopt_i],
\end{align*}
where the last inequality uses the $\KL$-computation above, and the fact that $g_{\bin \oplus i}$ and $g_{\bin}$ differ only on $I_i$. Hence, 
\begin{align*}
\frac{n\sigma^2\kl(1-\delta,\delta/n) }{8\epsilon^2} &\le \Exp_{\bin \unifsim \{0,1\}^n} \sum_{i=1}^n \Exp_{g_{\bin \oplus i}}[\stopt_i]  ~=~ \sum_{i=1}^n\left( \Exp_{\bin \unifsim \{0,1\}^n}  \Exp_{g_{\bin \oplus i}}[\stopt_i]\right)  \\
&= \sum_{i=1}^n \left(\Exp_{\bin \unifsim \{0,1\}^n}  \Exp_{g_{\bin}}[\stopt_i]\right)~=\Exp_{\bin \unifsim \{0,1\}^n}  \sum_{i=1}^n \Exp_{g_{\bin}}[\stopt_i]~.
\end{align*}
Bounding $\sum_{i=1}^n \Exp_{g_{\bin}}[\stopt_i]~ \le \Exp_{g_{\bin}}[\stopt]$ and $\kl(1-\delta,\delta/n) \gtrsim \log(n/\delta)$ concludes the proof.

\textbf{Lower bound when $\sigma$ is small:}
To conclude, we need to show that even when $\sigma$ is arbitrarilyy small (even zero), we still have the bounds $\Exp_{g}[\stopt] \gtrsim \Npack(\fst,2\epsilon)$ for every $g \in \calG(\fst,2\epsilon)$. To this end, fix $g\in \calG_{\fst,2\epsilon}$; we show $\Exp[\stopt_i] \ge 1/3$. Let $h$ be the alternative to $g$ in $\calG_{\fst,2\epsilon}$ which differs only on $I_i$, and let $\calB_i$ denote the event that $\Alg$ never samples in $I_i$. Note then that for any event $\calA$, $\Pr_{g}[\calA \cap \calB_i] = \Pr_{h}[\calA \cap \calB_i]$. Hence,
\begin{align*}
2\delta &\ge \Pr_g[\calA_g^c] + \Pr_h[\calA_h^c] ~\ge \Pr_g[\calA_g^c \cap \calB_i] + \Pr_h[\calA_h^c \cap \calB_i] \\
&\ge \Pr_{g}[(\calA_g^c \cup \calA_h^c)^c \cap \calB_i] =  \Pr_{g}[(\calA_g \cap \calA_h)^c \cap \calB_i] =  \Pr_g[\calB_i]~,
\end{align*}
where we used that $\calA_g \cap \calA_h = \emptyset$. Hence $\Exp[\stopt_i] \ge 1- \Pr_g[\calB_i] \ge 1 - 2\delta \ge 1/3$. 

\textbf{Lower bound when $\Npack(f,2\epsilon)  = 0$}. When $\Npack(f,2\epsilon) < 1$, we can consider the single alternative function $\widetilde{f}(x) = f(x) + 2\epsilon$. Since $|\widetilde{f}(x) - f(x)|  = 2\epsilon$ for all $x \in [0,1]$, the above arguments show that one needs at least $\gtrsim \max\{1, \frac{\sigma^2}{\epsilon^2} \log(1/\delta)\}$ samples to distinguish between $\widetilde{f}$ and $f$. 
\subsection{Proof of Proposition~\ref{prop:oracle}}
Recall that the construction of the packing in Theorem~\ref{thm:packing} in Section~\ref{sec:packing_proof} is constructed with $n = \Npack(f,\epsilon) + 1$ intervals of the form $\{[a_i,b_i]\}_{i = 1}^{n}$ with $\Delta(f,\epsilon) = \epsilon$. Define the interval $[a_0,b_0] = [0,a_1]$, and $[a_{n+1},b_{n+1}] = [b_n,1]$. The following fact is straightforward to verify using the construction in Section~\ref{sec:packing_proof}:
\begin{fact}\label{fact:oracle} The intervals $\{[a_i,b_i]\}_{i \in [n+1]\cup\{0\}}$ cover $[0,1]$, and satisfy $\Delta(f,[a_i,b_i]) \le \epsilon$.
\end{fact}
Let $\calX := \{a_i,b_i, \frac{a_i+b_i}{2}\}_{i \in [n+1] \cup \{0}$. We collect $\max\{1, \frac{8\sigma^2}{\epsilon^2}\log (|\calX|/2\delta))\}$-samples at each $x \in \calX$, and define $\fhat(x)$ to denote the empirical mean of these samples. We then define our test function to be
\begin{align*}
\psi := \I(\sup_{x \in \calX} |\fhat(x) -\fst(x)| \ge \epsilon/2)~.
\end{align*}
It now suffices to show that $\Pr_{\fst}[\psi \ne 0] \le \delta/2$ and $\Pr_{f}[\psi \ne 1] \le \delta/2$ for $f \in \class$ satisfying $\|f - \fst\|_{\infty} \ge 10 \epsilon$.
By standard sub-gaussian concentration, 
\begin{align}\label{eq:oracle_whp}
\forall f, \Pr_{f}[\sup_{x \in \calX} |\fhat(x) -f(x)| \ge \epsilon/2] \le \delta/2,
\end{align}
which immediately implies that $\Pr_{\fst}[\psi \ne 0] \le \delta/2$. To prove the other direction, it suffices to prove the following lemma:
\begin{lemma}\label{lem:oracle_covering_lem} If $f \in \class$ satisfies $\|f - \fst\|_{\infty} \ge 9\epsilon$, then there exists an $x \in \calX$ such that $|f(x) - \fst(x)| >  \epsilon$.
\end{lemma}
Indeed, by Lemma~\ref{lem:oracle_covering_lem}, the triangle inequality, and~\eqref{eq:oracle_whp} we have
\begin{align*}
\Pr_{f}[\psi \ne 1] &= \Pr_{f}[\sup_{x \in \calX} |\fhat(x) -\fst(x)| \le \epsilon/2] \le \Pr_{f}[\sup_{x \in \calX} |\fhat(x) -f(x)| \le \epsilon/2] \le \delta/2~.
\end{align*}
\begin{proof}[Proof of Lemma~\ref{lem:oracle_covering_lem}] We prove the contrapositive. Suppose that $\sup_{x \in \calX} |f(x) - \fst(x)| < \epsilon$, and let $z \in [0,1]$. Then $z \in [a_i,b_i]$ for some $i \in \{0,\dots,n+1\}$. Let $I = [a_i,b_i]$ and $m_i = (b_i+a_i)/2$. We then have that 
\begin{align*}
|f(z)-\fst(z)| &\le |\Sec[f,I](z) - \Sec[\fst, I](z)| + |f(z) - \Sec[f,I](z)| + |\fst(z) - \Sec(\fst(z))|\\
&\overset{(i)}{\le} |\Sec[f,I](z) - \Sec[\fst, I](z)| + 2\Delta(f,I) + 2\Delta(\fst,I) \quad (\text{Lemma~\ref{lem:mainLemmaSimple}}) \\
&\le |\Sec[f,I](z) - \Sec[\fst, I](z)| + 2|\Delta(f,I) - \Delta(\fst,I)| + 4\Delta(\fst,I)~\\
&\le |\Sec[f,I](z) - \Sec[\fst, I](z)| + 2|\Delta(f,I) - \Delta(\fst,I)| + 4\epsilon \quad \text{(Fact~\ref{fact:oracle})}.
\end{align*}
Lastly, we observe that $|\Delta(f,I) - \Delta(\fst,I)| \le  |f(m_i) - \fst(m_i)| + |\Sec[f,I](m_i) - \Sec[\fst, I](m_i)| \le \epsilon  + |\Sec[f,I](m_i) - \Sec[\fst, I](m_i)|$. Moreover, for all $t \in I$ (in particular $t = z,m_i$), we have $|\Sec[f,I](t) - \Sec[\fst, I](t)| \le \max\{|f(a_i) - \fst(a_i)|,|f(b_i) - \fst(b_i)|\}$, which is $< \epsilon$ by assumption. Thus, we can bound $|\Sec[f,I](z) - \Sec[\fst, I](z)| + 2|\Delta(f,I) - \Delta(\fst,I)| < 5\epsilon$. Putting things together, $|f(z)-\fst(z)| < 9\epsilon$, as needed.
\end{proof}

\label{sec:packing_proof}

\section{Structural Results about Convex Functions\label{l_infty_lemma}}
In this section, we introduce structural tools regarding convex functions, and use these tools to concludes the proof of the technical lemmas used above.  The first guarantee is that the error of secant approximation is \emph{monotone} in the following sense:
\begin{lemma}\label{secant_approx_grows} For any $x\in [a,b]\subset [c,d]$, $\Sec[f,[a,b]](x) \le \Sec[f,[c,d]](x)$. 
\end{lemma}
Next, we state a generalization of Lemma~\ref{lem:mainLemmaSimple}:
\begin{lemma}\label{somelemma}
Let $f:[0,1]\to\R$ be convex. For any $x,z \in (t_1,t_2) \subset [0,1]$, one has that
\begin{align*}
\Sec[f,[t_1,t_2]](x) - f(x) \ge \{\Sec[f,[t_1,t_2]](z) - f(z)\}  \cdot \min \left\{\frac{t_2-x}{t_2-z} , \frac{x-t_1}{z-t_1}\right\}~.
\end{align*}
\end{lemma}
We observe that Lemma~\ref{lem:mainLemmaSimple} follows as a corollary by choosing $t_1 = \xl$, $t_2 = \xr$, and $x = \xm$, and considering the maximum over $z \in [\xl,\xr]$. 

The remainder of the section is organized as follows. In Section~\ref{sec:lem:someLemma}, we prove Lemma~\ref{somelemma} and in Section~\ref{sec:ModulusUB}, we prove Lemma~\ref{ModulusUB}. We then introduce further technical lemmas in Section~\ref{sec:addition_lemmas}, which we use to prove Proposition~\ref{LambdaScaleProp} in Section~\ref{sec:LambdaScalePropProof}, and Lemma~\ref{lem:LB_omega_local} in~\ref{sec:lem:LB_omega_local}. The proof of Lemma~\ref{secant_approx_grows} is given in Section~\ref{sec:secant_approx_grows}.

\subsection{Proof of Lemma~\ref{somelemma}\label{sec:lem:someLemma}}  Note that adding an affine function to $f$ does not change the value of $\Sec[f,[t_1,t_2]](x) - f(x)$. 
Thus, we may assume that $f(t_1) = f(t_2) = 0$.
Without loss of generality, we may also take $t_1 = 0$ and $t_2 = 1$. With these simplifications, $\Sec[f,[0,1]]  = \Sec[f,[t_1,t_2]] = 0$, and hence our goal is show that
	\begin{align*}
	-f(x) \ge \begin{cases} \frac{x}{z} \cdot (-f(z)) & 0 \le x \le z \le 1 \\
	\frac{1-x}{1-z}\cdot (-f(z)) & 0 \le z \le x \le 1
	\end{cases}~,
	\end{align*}
	or equivalently, that
	\begin{align*}
	f(z) \ge \begin{cases} \frac{z}{x} \cdot f(x) & 0 \le x \le z \le 1 \\
	\frac{1-z}{1-x} \cdot f(x) & 0 \le z \le x \le 1 
	\end{cases}~.
	\end{align*}
	To this end, fix a subgradient $g \in \partial f(x)$ and some $z \ge x$. By the definition of the subgradient, it holds that
	\begin{align*}
	f(x) + g(t-x) \le f(t) \quad \forall t \in [0,1]~.
	\end{align*}
	By choosing $t = 0$ and $t = z$ in the above display, we verify that (a) $f(x) + g(0-x) \le 0$, and (b) $f(x) + g(z - x) \le f(z)~$. Combining $(a)$ and $(b)$, and noting that $z \ge x$, we find
	\begin{equation}
	\begin{aligned}\label{zgex_ineq}
	f(z) &\overset{(b)}{\ge} f(x) + g(z-x) ~\overset{(a)}{\ge} f(x) + \frac{1}{x}f(x)(x-z) \\
	&= f(x)\left( 1 + \frac{z-x}{x}\right) \quad = \frac{x}{z} \cdot f(x)~,
	\end{aligned}
	\end{equation}
	as needed. On the other hand, suppose $x \ge z$. Noting that the function  $\widetilde{f}(t) = f(1-t)$ is convex and satisfies $\widetilde{f}(0) =  \widetilde{f}(1) = 0$, we have
	\begin{align*}
	 f(z) = \widetilde{f}(1-z) \overset{\eqref{zgex_ineq}}{\ge} \left(\frac{(1-z)}{(1-x)}\right)\widetilde{f}(1-x) = \left(\frac{(1-z)}{(1-x)}\right)f(x)~.
	\end{align*}

\subsection{Proof of Lemma~\ref{ModulusUB}\label{sec:ModulusUB}}
For any $\tau = [-t,t]$, we have
\begin{align}
\quad&\Sec[f,[x-t,x+t]](x+\tau) - f(x+\tau) \nonumber \\
\overset{(\text{Lemma~\ref{somelemma}})}{\ge}& \min\{\frac{t-\tau}{t},\frac{t + \tau}{t}\}\cdot( \Sec[f,[x-t,x+t]](x) - f(x)) \nonumber\\
=& \min\{1 - \frac{\tau}{t},1 + \frac{\tau}{t}\}\cdot \Delta(f,x,t) ~=~  (1 - \frac{|\tau|}{t})\cdot \epsilon \label{eq:approximation_perturbed_LB}~.
\end{align}
First suppose that $\tau \geq 0$, then
\begin{align*}
\Delta(f,x+\tau,t+\tau) &= \Sec[f,[x-t,x+t+2\tau]](x+\tau) - f(x+\tau) \\
&\overset{(\text{Lemma}~\ref{secant_approx_grows})}{\ge} \Sec[f,[x-t,x+t]](x+\tau) - f(x+\tau) \\
&\overset{\eqref{eq:approximation_perturbed_LB}}{\ge} \epsilon (1 - \frac{|\tau|}{t}).
\end{align*}
On the other hand, if $\tau \leq 0$, then similarly we have
\begin{align*}
\Delta(f,x+\tau,t+|\tau|) = \Sec[f,[x-t+2\tau,x+t]](x+\tau) - f(x+\tau) \geq  \epsilon (1 - \frac{|\tau|}{t}). 
\end{align*}
By definition, and combining the above pieces, we get
\begin{align*}
\omega(f,x+\tau,\epsilon(1 - \frac{|\tau|}{t})) = \inf\{s : \Delta(f,x+\tau,s) \geq \epsilon(1 - \frac{|\tau|}{t}) \} \leq t + |\tau|.
\end{align*}

\subsection{Additional Structural Lemmas \label{sec:addition_lemmas}}
Before continuing, we state three additional structural results that we shall need throughout. First, we observe that the following secant approximation functions are monotone:
\begin{lemma}\label{monotone_lem}
	For any convex function $f:[0,1] \to \R$, the functions $t\mapsto\Delta(f,x,t)$, $t\mapsto\Delta(f,x+t,t)$ and $t\mapsto\Delta(f,x-t,t)$ (defined on the appropriate domains) are all non-decreasing in $t \ge 0$. 
\end{lemma}
\begin{proof}
The mononoticity of $t \mapsto \Delta(f,x,t)$ is a consequence of the monotonicity of secant approximations, Lemma~\ref{secant_approx_grows}. 
Here, we will prove that $t\mapsto\Delta(f,x+t,t)$ is non-decreasing; that $t\mapsto\Delta(f,x-t,t)$ is non-decreasing will follow by a similar argument. 
Write $\Delta(f,x+t,t) = \frac{f(x)+f(x+2t)}{2} - f(x+t)$. Since continuously differentiable convex functions are dense in class, we may assume $f'$ exists. Thus, $\frac{d}{dt} \Delta(f,x+t,t) = 2\cdot \frac{1}{2}f'(x+2t) - f'(x+t) = f'(x+2t) - f'(x+t)$, which is nonnegative by convexity.
\end{proof}

The next result states that the continuity modulus can be regarded as the inverse function of the secant error function $\Delta$, in the following sense:
\begin{lemma}\label{eps_f_u_lem} For any $\epsilon > 0$, and $x \in [\tleft(f,\epsilon),1 - \tright(f,\epsilon)]$, 
$\omega(f,x,\epsilon)$ is equal to the unique $t \in [0,x\wedge(1-x)]$ satisfying $\Delta(f,x,t) = \epsilon$. 
\end{lemma}
	\begin{proof} We may assume without loss of generality that $x \in [\tleft(f,\epsilon),1/2]$, and that $\min\left\{t \in \left[0,\min\left\{x,1-x\right\}\right]: \Delta\left(f,x, t\right) \ge   \epsilon   \right\}$. We first show that $\Delta(f,x,\omega(f,x,\epsilon)) = 0$. Suppose otherwise. Then by continuity of $\Delta(f,x,\cdot)$, it must be the case that $x = \omega(f,x,\epsilon)$ and $\Delta(f,x,x) < \epsilon$. But a contradiction arises from  $\epsilon~\le~\Delta(f,\tleft(\epsilon),\tleft(\epsilon))$ $\le\Delta(f,x,x) $, where the first (resp. second) inequality uses continuity (resp. montonicity,  Lemma~\ref{monotone_lem}) of  $t\mapsto \Delta(f,t,t)$.

	Now we show that $\omega(f,x,\epsilon)$ is equal to the unique $t \in [0,x\wedge(1-x)]$ satisfying $\Delta(f,x,t) = \epsilon$. Since $t \mapsto \Delta(f,x,t)$ is monotone, it suffices to show that $t \mapsto \Delta(f,x,t)$ is strictly increasing from the left at $t = \omega(f,x,\epsilon)$. This is follows because $t \mapsto \Delta(f,x,t)$ is convex, strictly positive at $t = \omega(f,x,\epsilon)$, and $\Delta(f,x,0)$. 
	\end{proof}
	Lemmas~\ref{monotone_lem} and~\ref{eps_f_u_lem} are used in Section~\ref{sec:packing_proof}, which proves the packing given in Theorem~\ref{thm:packing}.
	Next, we have a  `change-of-scale' lemma, whose proof is at the heart of Proposition~\ref{LambdaScaleProp} and Lemma~\ref{lem:LB_omega_local}:
\begin{lemma}\label{unif_continuity_omega} 
	For any $0 < \epsilon' \le \epsilon$ and $x \in [\tleft(f,\epsilon),1-\tright(f,\epsilon)]$, we have
	\begin{align*}
	\frac{\epsilon'}{\epsilon}\omega(f,x,\epsilon) \le \omega(f,x,\epsilon') \le \omega(f,x,\epsilon)~.
	\end{align*}
\end{lemma}
\begin{proof} Fix $0 < \epsilon' \le \epsilon$, $x \in [\tleft(f,\epsilon),1 - \tright(f,\epsilon)]$, which implies that $\omega(f,x,\epsilon) \le x \wedge (1-x)$. Let $\phi(t) :=  \Delta(f,x,t) = (f(x-t)+f(x+t))/2 -f(x)$, which is defined for $t \in [0, x \wedge (1-x)]$, convex, and satisfies $\phi(0)=0$. A standard computation (Lemma~\ref{zero_cvx_lemma}) shows that, for any $t' \le t$, one has $\phi(t') \le \frac{t'}{t} \phi(t)$. Hence, 
\begin{align*}
\epsilon' \overset{(i.a)}{=} \phi(\omega(f,x,\epsilon')) \overset{(ii)}{\ge} \frac{\omega(f,x,\epsilon')}{\omega(f,x,\epsilon)} \phi(\omega(f,x,\epsilon)) \overset{(i.a)}{=} \epsilon\cdot \frac{\omega(f,x,\epsilon')}{\omega(f,x,\epsilon)}~,
\end{align*}
where $(i.a)$ and $(i.b)$ are by Lemma~\ref{eps_f_u_lem}, and $(ii)$ uses the fact that $\omega(f,x,\epsilon') \le \omega(f,x,\epsilon)$ (this is immediate from the definition of $\omega$).
\end{proof}
\subsection{Proof of Proposition~\ref{LambdaScaleProp}\label{sec:LambdaScalePropProof}}
	Let $c \in (0,1)$, and observe that $\tleft(f,c\epsilon) \le \tleft(f,\epsilon)$, and similarly for $\tright$. Thus, 
	\begin{align*}
	& \Lamavg(f, c\epsilon) = \int_{\tleft(f,c\epsilon)}^{1-\tright(f,c\epsilon)} \omega^{-1}(f,x,c\epsilon)dx\\
	&= \int_{\tleft(f,c\epsilon)}^{\tleft(f,\epsilon)} \omega^{-1}(f,x,c\epsilon)dx \\
	&+ \int_{\tleft(f,\epsilon)}^{1-\tright(f,\epsilon)} \omega^{-1}(f,x,c\epsilon)dx + \int_{1-\tright(f,\epsilon)}^{1-\tright(f,c\epsilon)} \omega^{-1}(f,x,c\epsilon)dx~.
	\end{align*}
	By Lemma~\ref{unif_continuity_omega}, we have
	\begin{align*}
	[c,1]\int_{\tleft(f,\epsilon)}^{1-\tright(f,\epsilon)} \omega^{-1}(f,x,c\epsilon)dx \ni \int_{\tleft(f,\epsilon)}^{1-\tright(f,\epsilon)} \omega^{-1}(f,x,\epsilon)dx = \Lamavg(f, \epsilon)~.
	\end{align*}
	Next, let $x \in [\tleft(f,c\epsilon),\tleft(f,\epsilon)]$; we show that $\omega(f,x,c\epsilon) \ge cx$. Indeed, let  $\epsilon^*:= \Delta(f,x,x)$. 
	Since  $\tleft(f,\epsilon^*) \le x \le \tleft(f,\epsilon)$ and $\tleft(f,\cdot)$ is monotone, we have $\epsilon^* \le \epsilon$.
	By definition, both $\omega(f,x,c\epsilon) \le x$  and $\omega(f,x,\epsilon^*) \le x$. Hence, Lemma~\ref{unif_continuity_omega} and the bound $\epsilon_* \le \epsilon$ imply
	\begin{align*}
	\omega(f,x,c\epsilon) \ge \frac{c\epsilon}{\epsilon^*} \omega^{-1}(f,x,\epsilon^*) \ge x\cdot\frac{c\epsilon}{\epsilon^*} \ge cx,~
	\end{align*}
	as needed. Hence,
	\begin{eqnarray*}
	 \int_{\tleft(f,c\epsilon)}^{\tleft(f,\epsilon)} \omega^{-1}(f,x,c\epsilon)dx \in [1,\frac{1}{c}]\int_{\tleft(f,c\epsilon)}^{\tleft(f,\epsilon)}\frac{dx}{x} = [1,\frac{1}{c}] \cdot \log\frac{\tleft(f,\epsilon)}{\tleft(f,c\epsilon)}~.
	\end{eqnarray*}
	The case $x \in [1-\tright(f,\epsilon),1-\tright(f,c\epsilon)]$ similarly yields
	\begin{eqnarray*}
	\int_{1-\tright(f,\epsilon)}^{1-\tright(f,c\epsilon)} \omega^{-1}(f,x,c\epsilon)dx \in [1,\frac{1}{c}]\log\frac{\tright(f,\epsilon)}{\tright(f,c\epsilon)}~.
	\end{eqnarray*}
	Putting everything together,
	\begin{eqnarray*}
	\Lambda(f, c\epsilon) \in [1,\frac{1}{c}]\left\{\Lambda(f,\lambda\epsilon) + \log \frac{\tright(f,\epsilon)\tleft(f,\epsilon)}{\tright(f,c\epsilon)\tleft(f,c\epsilon)}) \right\}.
	\end{eqnarray*}

\subsection{Proof of Lemma~\ref{lem:LB_omega_local}\label{sec:lem:LB_omega_local}}

We may assume without loss of generality that $\tau \in [0,1]$. 
For ease of notation, set 
\begin{align*}
t &= \omega(f,x,\epsilon) \qquad u = x + \tau t \qquad \widetilde{\epsilon} = \Delta(f,u,(1-\tau)t)~,
\end{align*} 
noting that $\widetilde{\epsilon}$ is the secant approximation bias on the interval $[u - (1-\tau)t,u+(1-\tau)t]$.
Since $u+(1-\tau)t = x+t$ and $[u - (1-\tau)t,u+(1-\tau)t] \subseteq [x-t,x + t]$, we have that 
\begin{align*}
\widetilde{\epsilon} &= \Delta(f,[u - (1-\tau)t,u+(1-\tau)t]) \\
&\overset{(\text{Lemma}~\ref{secant_approx_grows})}{\leq} \sup_{y \in [x-t,x+t]} \Sec[f,[x-t,x+t]](y) - f(y) \\
&\overset{(\text{Lemma}~\ref{lem:mainLemmaSimple})}{\leq} 2 \Delta(f,[x-t,x+t]) = 2 \epsilon.
\end{align*}

First, if $\widetilde{\epsilon} \leq \epsilon$ then
\begin{eqnarray}
\omega(f,u,\epsilon) \overset{(\text{Lemma}~\ref{unif_continuity_omega})}{\ge} \omega(f,u,\widetilde{\epsilon}) \overset{(\text{Lemma}~\ref{eps_f_u_lem})}{=} (1 - \tau)t.
\end{eqnarray}
On the other hand, if $\epsilon < \widetilde{\epsilon} \leq 2 \epsilon$ then
\begin{align*}
\omega(f,u,\epsilon) \overset{\text{Lemma}~\ref{unif_continuity_omega}}{\ge} \frac{\epsilon}{\widetilde{\epsilon}} \, \omega(f,u,\widetilde{\epsilon}) \geq \frac{1}{2} \, \omega(f,u,\widetilde{\epsilon}) \overset{(\text{Lemma}~\ref{eps_f_u_lem})}{=} \frac{(1 - \tau)t}{2}.
\end{align*}
In either case $\omega(f,u,\epsilon)  \ge \frac{(1 - \tau)t}{2}$, which conclude the proof.

\subsection{Proof of Lemma~\ref{secant_approx_grows}\label{sec:secant_approx_grows}}
We begin with a simple technical lemma.
\begin{lemma}\label{zero_cvx_lemma} Let $\varphi:[0,t] \to 1$ be a convex function satisfying $\varphi(0) = 0$. Then for all $c \in [0,1]$, $\varphi(ct) \le c\varphi(t)$. 
\end{lemma}
\begin{proof} 
 $c\varphi(t) = c\varphi(t) + (1-c)\varphi(0) \le \varphi(ct + (1-c)0) \varphi(ct)$, where the first equality uses $\varphi(0) = 0$ and the second uses convexity. 
\end{proof}
We are now ready to prove Lemma~\ref{secant_approx_grows}:
\begin{proof}[Proof of Lemma~\ref{secant_approx_grows}] It suffices to prove that this is the case when $a = c$ or $d = b$, since then $\Sec[f,[a,b]](x) \le \Sec[f,[c,b]](x) \le  \Sec[f,[c,d]](x)$. We assume without loss of generality that $a = c$. Then, it suffices to show that, for $t \ge x$, the map $t \mapsto \Sec[f,[a,a+t]](x)$ is non-decreasing. We have
\begin{align*}
\Sec[f,[a,a+t]](x) = \frac{f(a)(a+t-x) + (x-a)f(a+t)}{t} = \frac{\varphi(t)}{t}~,
\end{align*}
where $\varphi(t) = f(a)(a+t-x) + (x-a)f(a+t)$. Since $\varphi$ is convex (sum of affine function and convex function as $x \ge a$, and $t \mapsto f(a+t)$ is convex), and $\varphi(0) = 0$, we conclude by Lemma~\ref{zero_cvx_lemma} that $\frac{\varphi(t)}{t}$ is non-decreasing, as needed. 
\end{proof}





\section{Empirical Results}
In this section, we validate our theoretical results through empirical comparisons of active and passive sampling using simulated data and data drawn from the behavioral literature.
In all experiments, a query at $x \in [0,1]$ results in an observation $y \overset{i.i.d.}{\sim} \mathcal{N}(f(x),\sigma^2)$ where $\sigma$ depends on the experiments and is known to the algorithm. 
We construct our confidence intervals $\phi(t,\delta)$ using \citet[Theorem 8]{kaufmann2016complexity}, scaled by $\sigma$. Further implementation details are described in Section~\ref{sec:implementation}.

\subsection{Piecewise Linear Function}
We begin by comparing the performance of the active and passive methods on a piecewise linear function, $f(x) = \max\{1 - 5x,0\}$, over the domain $x \in [0,1]$. We consider the noise level $\sigma^2 = .01$. 
As discussed in Remark~\ref{rem:piecewiselinear}, theory predicts that, up to logarithmic factors, the error incurred by passive sampling scales as $\|\fhat - \fst\|_{\infty} \sim n^{-1/3}$, whereas active sampling attains the parameteric rate $\|\fhat - \fst\|_{\infty} \sim n^{-1/2}$. As a benchmark, we plot a passive algorithm based on constrainted least squares (see, e.g.~\cite{dumbgen2004consistency}), and also plot the error incurred by sampling according to an ``oracle allocation'', which samples $f$ at the endpoints $\{0,1\}$, as well as the inflection point $x = 0.2$. The implementation details are deferred to the end of the section.

\begin{figure}[H]
\floatbox[{\capbeside\thisfloatsetup{capbesideposition={right,center},capbesidewidth=4cm}}]{figure}[\FBwidth]
{\caption{Comparison of active, passive and oracle performance on $f(x) = \max\{1 - 5x, 0\}$. The $x$-axis is the number of samples taken, and the $y$-axis is $L_\infty$ error. Dotted lines denote a least-squares trendline. A slope of $-p$ suggests a rate of $(\# \mathrm{samples})^{-p}$.}\label{fig:piecewise_linear}}
{\includegraphics[width=0.5\textwidth]{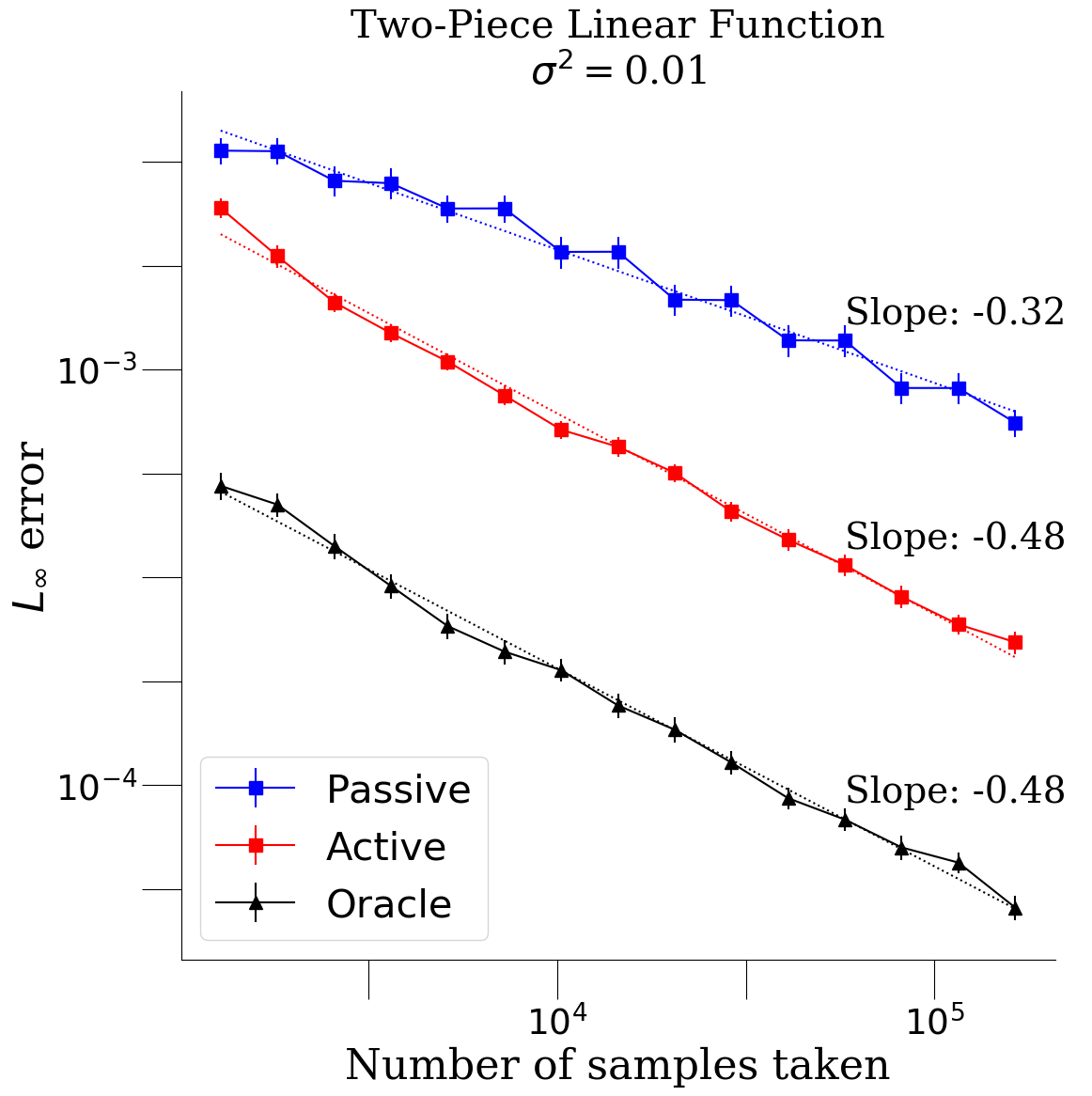}}
\end{figure}

Figure~\ref{fig:piecewise_linear} corroborates our theoretical predictions. The dotted trend lines correspond to a least-squares fit to the logarithm of the $x$ and $y$ coordinates, so that the displayed slopes approximate the exponent in the rate at which the errors decay. In particular, we see that the slope of the line corresponding to passive sampling is close to $-\frac{1}{3}$ indicating a rate approximately equal to $n^{-1/3}$, and the slopes for the active and oracle methods are close to $-\frac{1}{2}$. Observe that the oracle method still significantly outperforms the active sampling algorithm, perhaps explained by the additional $\log(1/\epsilon)$ superfluous locations the active procedure samples at to achieve $\epsilon$-error relative to the oracle method. 

\subsection{Data Derived Function}
Next, we evaluate the performance of our active algorithm on a convex function derived from real data. 250 participants  were asked to choose between a hypothetical reward of \$100 given immediately and a reward of \$115 given at a time $x = 0, 1, 2, \dots, 64$ days in the future (times were randomized and rescaled to be in $[0,1]$). 
We fit a convex function to this data using least squares and sampled from it as above; the function is displayed in Figure~\ref{fig:data_f}.
\begin{figure}[H]
\floatbox[{\capbeside\thisfloatsetup{capbesideposition={right,center},capbesidewidth=4cm}}]{figure}[\FBwidth]
{\caption{Data-derived discount function. Here, the $x$ value corresponds for the days for which the reward is delayed, and $y$ is the fraction of the population who would accept the delay for greater monetary reward.}\label{fig:data_f}}
{\includegraphics[width=0.5\textwidth]{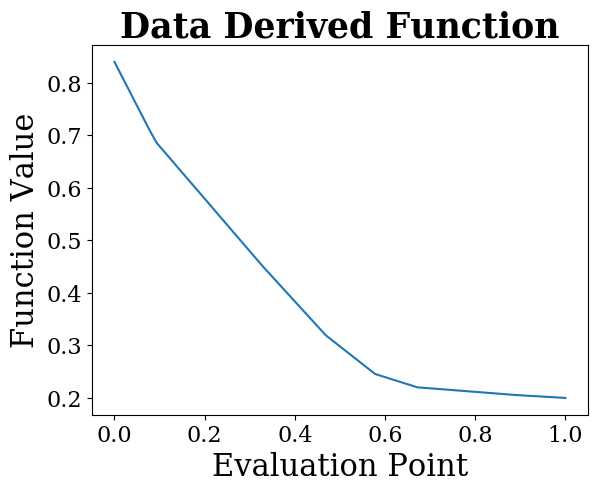}}
\end{figure}

\begin{figure}[H]
\floatbox[{\capbeside\thisfloatsetup{capbesideposition={right,center},capbesidewidth=4cm}}]{figure}[\FBwidth]
{\caption{Comparison of active and passive performance on the function depicted in Figure~\ref{fig:data_f}. The $x$-axis is the number of samples taken, and the $y$-axis is $L_\infty$ error. Dotted lines denote a least-squares trendline. }\label{fig:data_performance}}
{\includegraphics[width=0.5\textwidth]{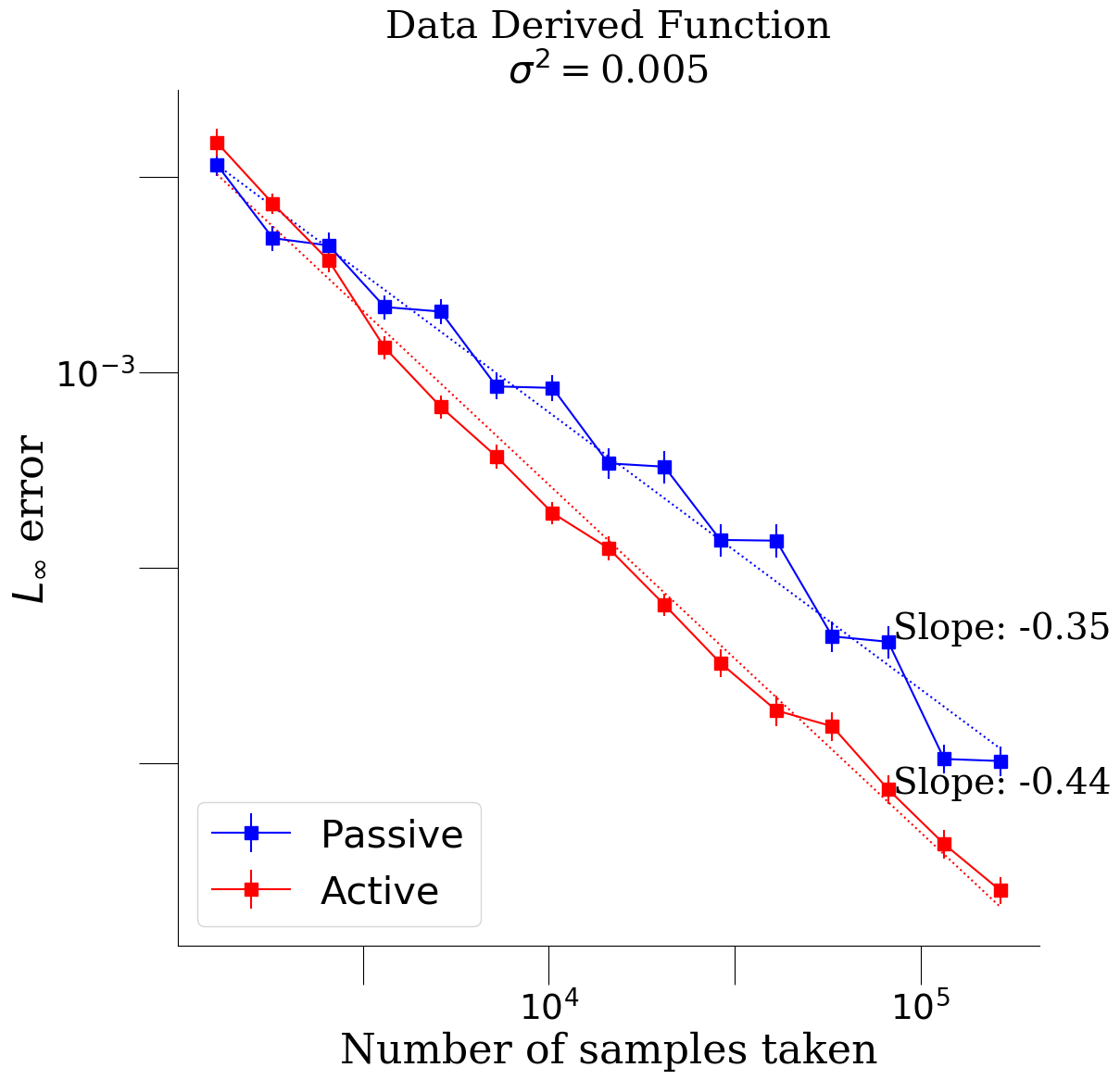}}
\end{figure}
In Figure~\ref{fig:data_performance}, we compare the performance of the active and passive algorithms for the function $f$ depicted in Figure~\ref{fig:data_f}. Again, the dotted trend lines correspond to a least-squares fit to the logarithmic of the $x$ and $y$ coordinates. We find that the passive algorithm appears to obtain a rate of $n^{-1/3}$, and the error of the active algorithm has a scaling closer to that of the parametric rate. 

For insight into why the active algorithm fares better on this $f$, we can examine the oracle allocations, as constructed in Proposition~\ref{prop:oracle}. We find that at higher levels of granularity, $f$ is well approximated by a piecewise linear function, and thus an oracle would only sample at a few, key design points (Figure~\ref{fig:data_oracle_alloc}):
\begin{figure}[H]
\floatbox[{\capbeside\thisfloatsetup{capbesideposition={right,center},capbesidewidth=4cm}}]{figure}[\FBwidth]
{\caption{A plot of the design points correspond to the oracle allocation constructed in Proposition~\ref{prop:oracle}, for granularities $\epsilon \in \{.01,.001\}$. Bar height is equal to $1000$ divided by the number of design points. }\label{fig:data_oracle_alloc}}
{\includegraphics[width=0.5\textwidth]{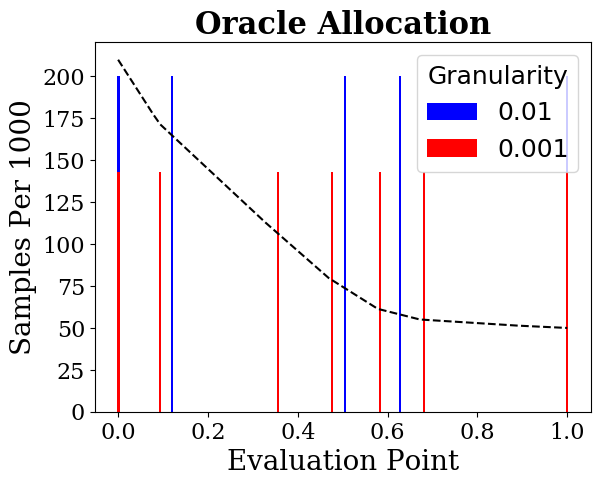}}
\end{figure}
Notice that as the granularity decreases, the oracle allocation refines its design points, dividing the function into regions in which the piecewise-linear approximation holds to a higher accuracy. 
\subsection{Implementation\label{sec:implementation}} The passive algorithm samples at design points $x_t$ along a dyadic sequence $\dyads = (0,1,1/2,1/4,3/4,1/8,\dots)$. The algorithm then estimates $\fst$ using the following constrained least squares problem. 
\begin{align}\label{eq:least_squares}
\fhat \in \arg\min_{f \in \class} \sum_{t =1}^{T} \|F(x_t) - \fhat\|_2^2~.
\end{align}
We found in practice that this least squares minimization performs surprising well in practice, drastically outperforming projections in the $\infty$-norm. We used this observation to modify the implementation of our active algorithm. At each round, the active algorithm alternates between sampling a design points $x \in \{x_{\ell(I^*)},x_{m(I^*)},x_{r(I^*)}\} $ as in Algorithm~\ref{AlgorithmNoise}, and sampling dyadic sequences supported on $I^*$ given by $x_{\ell(I^*)} + |I^*|\dyads$. We then return $\fhat$ using the constrained least-squares problem in~\eqref{eq:least_squares}. In addition, our implementation makes use of the sharper upper- and lower-confidence bounds described in Remark~\ref{remark:splitting}.


\clearpage
\bibliographystyle{plainnat}
\bibliography{manuscript}

\begin{thebibliography}{14}
\providecommand{\natexlab}[1]{#1}
\providecommand{\url}[1]{\texttt{#1}}
\expandafter\ifx\csname urlstyle\endcsname\relax
  \providecommand{\doi}[1]{doi: #1}\else
  \providecommand{\doi}{doi: \begingroup \urlstyle{rm}\Url}\fi

\bibitem[Bellec et~al.(2018)]{bellec2018sharp}
Pierre~C Bellec et~al.
\newblock Sharp oracle inequalities for least squares estimators in shape
  restricted regression.
\newblock \emph{The Annals of Statistics}, 46\penalty0 (2):\penalty0 745--780,
  2018.

\bibitem[Boucheron et~al.(2013)Boucheron, Lugosi, and
  Massart]{boucheron2013concentration}
St{\'e}phane Boucheron, G{\'a}bor Lugosi, and Pascal Massart.
\newblock \emph{Concentration inequalities: A nonasymptotic theory of
  independence}.
\newblock Oxford University Press, 2013.

\bibitem[Cai et~al.(2013)Cai, Low, Xia, et~al.]{cai2013adaptive}
T~Tony Cai, Mark~G Low, Yin Xia, et~al.
\newblock Adaptive confidence intervals for regression functions under shape
  constraints.
\newblock \emph{The Annals of Statistics}, 41\penalty0 (2):\penalty0 722--750,
  2013.

\bibitem[Castro et~al.(2005)Castro, Willett, and Nowak]{castro2005faster}
Rui Castro, Rebecca Willett, and Robert Nowak.
\newblock Faster rates in regression via active learning.
\newblock \emph{Advances in Neural Information Processing Systems}, 18, 2005.

\bibitem[Chatterjee(2016)]{chatterjee2016improved}
Sabyasachi Chatterjee.
\newblock An improved global risk bound in concave regression.
\newblock \emph{Electronic Journal of Statistics}, 10\penalty0 (1):\penalty0
  1608--1629, 2016.

\bibitem[D{\"u}mbgen et~al.(2004)D{\"u}mbgen, Freitag, and
  Jongbloed]{dumbgen2004consistency}
Lutz D{\"u}mbgen, Sandra Freitag, and Geurt Jongbloed.
\newblock Consistency of concave regression with an application to
  current-status data.
\newblock \emph{Mathematical Methods of Statistics}, 13:\penalty0 69--81, 2004.

\bibitem[D{\"u}mbgen et~al.(2003)]{dumbgen2003optimal}
Lutz D{\"u}mbgen et~al.
\newblock Optimal confidence bands for shape-restricted curves.
\newblock \emph{Bernoulli}, 9\penalty0 (3):\penalty0 423--449, 2003.

\bibitem[Fisher(1937)]{fisher1937design}
Ronald~Aylmer Fisher.
\newblock \emph{The design of experiments}.
\newblock Oliver And Boyd; Edinburgh; London, 1937.

\bibitem[Frederick et~al.(2002)Frederick, Loewenstein, and
  O'donoghue]{frederick2002time}
Shane Frederick, George Loewenstein, and Ted O'donoghue.
\newblock Time discounting and time preference: A critical review.
\newblock \emph{Journal of economic literature}, 40\penalty0 (2):\penalty0
  351--401, 2002.

\bibitem[Green and Myerson(2004)]{green2004discounting}
Leonard Green and Joel Myerson.
\newblock A discounting framework for choice with delayed and probabilistic
  rewards.
\newblock \emph{Psychological Bulletin}, 130\penalty0 (5):\penalty0 769--792,
  2004.

\bibitem[Guntuboyina and Sen(2015)]{guntuboyina2015global}
Adityanand Guntuboyina and Bodhisattva Sen.
\newblock Global risk bounds and adaptation in univariate convex regression.
\newblock \emph{Probability Theory and Related Fields}, 163\penalty0
  (1-2):\penalty0 379--411, 2015.

\bibitem[Kaufmann et~al.(2016)Kaufmann, Capp{\'e}, and
  Garivier]{kaufmann2016complexity}
Emilie Kaufmann, Olivier Capp{\'e}, and Aur{\'e}lien Garivier.
\newblock On the complexity of best-arm identification in multi-armed bandit
  models.
\newblock \emph{The Journal of Machine Learning Research}, 17\penalty0
  (1):\penalty0 1--42, 2016.

\bibitem[Korostelev(1999)]{korostelev1999minimax}
Alexander Korostelev.
\newblock On minimax rates of convergence in image models under sequential
  design.
\newblock \emph{Statistics \& Probability Letters}, 43\penalty0 (4):\penalty0
  369--375, 1999.

\bibitem[Zhu et~al.(2016)Zhu, Chatterjee, Duchi, and Lafferty]{zhu2016local}
Yuancheng Zhu, Sabyasachi Chatterjee, John Duchi, and John Lafferty.
\newblock Local minimax complexity of stochastic convex optimization.
\newblock \emph{Advances in Neural Information Processing Systems}, 29, 2016.

\end{thebibliography}
\clearpage
\appendix




\section{Additional Remarks}

\begin{remark}[Gap Between Upper and Lower Bounds]\label{rem:gap_upper_lower}
	Let $f$ be a $k$-piecewise linear function $f$. Then there exists a set of $k$ intervals $\calI = \{I_i\}_{1 \le i \le k}$ such that $f$ is linear on each interval; measuring $f$ at $\{\xl,\xr,\xm\}_{I \in \calI}$ would be enough to estimate $f$ with zero error over $[0,1]$.
	Hence, we must have $\underline{N}(f,\epsilon) \lesssim k$.
	On the other hand, $\Lamavg(f,\epsilon) \approx k \log(1/\epsilon)$, and indeed one can show that for a $k$-piecewise linear function, $\log(\omega_{\max}/\omega_{\min}) \approx \log(1/\epsilon)$, yielding the necessary cancelation.
	As with the term $\log\frac{\tleft(f,\epsilon)\tright(f,\epsilon)}{\tleft(f,c\epsilon)\tright(f,c\epsilon)}$, $\log(\omega_{\max}/\omega_{\min})$ scales at most as $\log(1/\epsilon)$ for most reasonable functions, and can be bounded by $\log(\max_{x \in [0,1]}f''(x)/\min_{x \in [0,1]} f''(x))$ for any twice-differentiable function $f$.
	Overall, we conjecture that the true sample complexity lies closer to $\underline{N}(f,\epsilon)$, because in the noiseless setting, one can approximate left- and right-derivatives of $f$ to arbitrary accuracy using just two points. 
	This makes it possible to learn a 2-piecewise linear function with a constant number of function evaluations, rather than the $O(\log(1/\epsilon))$ implied by $\Lambda(f,\epsilon)$.
\end{remark}

\begin{remark}[Sub-Optimality of Non-Uniform Designs]\label{rem:non_unif_subopt} In this remark, we argue that although non-uniform designs can improve upon uniform designs for some functions of interest, in general they provide no benefit. 

To present the argument, we begin with by recalling the proof of Theorem~\ref{thm:passive_lb}. We argued that given $f \in \class$ and $x_0 \in [\tleft(f,2\epsilon), 1 - \tright(f,2\epsilon))]$, then unless a design collects  $\gtrsim (1 + \frac{\sigma^2}{\epsilon^2})\log(1/\delta)$ samples in the interior of the interval $I_{0}:= [x_0-\omega(f,x_0,2\epsilon),x_0+\omega(f,x_0,2\epsilon)]$, the alternative function
\begin{align*}
\widetilde{f}(x) := f(x) + \I(x \in I_0) \cdot \left(\Sec[f,I_0](x) - f(x)\right)
\end{align*}
satisfies $\sup_{g \in \{f,\widetilde{f}\}}\Pr_{\Alg,g}[\|\widehat{f} - g\|_{\infty} \ge \epsilon] \ge \delta$. In Theorem~\ref{thm:passive_lb}, we chose $x_0 = x_* := \arg\min\{ \omega(f,x,2\epsilon) |x\in [\tleft(f,2\epsilon), 1 - \tright(f,2\epsilon))]\}$, and defined $I_*$ analogously. 
By uniformity, the design could not oversample in the interior of $I_*$, and thus the total number of samples collected, up to constants, would need to be
\begin{align*}
\frac{1}{|I_*|}(1 + \frac{\sigma^2}{\epsilon^2})\log(1/\delta) \gtrsim \Lammax(f,2\epsilon)\cdot (1 + \frac{\sigma^2}{\epsilon^2})\log(1/\delta)~.
\end{align*}
Without uniformity, we cannot rule out that the design concentrates its samples in $I_*$. Nevertheless, we can show that there is a function ``similar'' to $f$ which incurs the same lower bound. For ease, assume that $f$ is right differentiable at $x = 1$ and left differentiable at $x = 0$.%
\footnote{In general, $f$ is convex and thus right differentiable at $1 - \eta$ and left differentiable at $\eta$ for all $\eta > 0$. Hence, one can modify $f$ with a linear extension to $[0,\eta]$ and $[1-\eta,1]$ to ensure that this condition holds.}
Letting $\partial_-$ and $\partial_+$ denote the right- and left-derivative, we can define the shift function as 
\begin{align*} \fshift(x) := 
\begin{cases} f(x-t) & x \in [0 \vee t, 1 \wedge 1 + t] \\
f(1 ) + (x - 1 - t)\partial_{-} f(1) & x \ge 1 - t, t < 0, \\
f(0) + (t - x)\partial_{+} f(0) & x \le t, t > 0
\end{cases}.
\end{align*}
We observe that $\fshift$ is convex, and if $x^*$ is as above, then for any $t \in (x^* + \omega(f,x^*,2\epsilon) - 1,   x^* - \omega(f,x^*,2\epsilon))$, one can verify that
\begin{align*}
\omega(\fshift, x^* + t, 2\epsilon) = \omega(f,x^*,2\epsilon)~.
\end{align*}
Then, for any interval $\Ibad = [a,b] \subset [0,1]$ for which $|\Ibad| \ge 2I^*_{\ell}$, there exists a $\tbad \in \R$ such that $I^*_{\leftarrow \tbad} := I^* + \tbad |I^*| \subset \Ibad$. Hence, if  $\Exp[|x_t : x_t \in \Ibad|] \ll \sigma^2 \log(1/\delta)$, then even if the passive design can estimate $f$ correctly, it will fail to distinguish between the shifted function $f_{\leftarrow \tbad}$ and the shifted alternative:
\begin{align*}
\widetilde{f}_{\leftarrow \tbad}(x) := f_{\leftarrow \tbad}(x) +  \I(x \in I^*_{\leftarrow \tbad})(\Sec[f_{\leftarrow \tbad},I^*_{\leftarrow \tbad}](x) - f_{\leftarrow \tbad}(x)).
\end{align*}
As a consequence, any algorithm which is $\delta$-correct for all $f$ shifts $f_{\leftarrow t}$, and alternatives defined above must collect at least $\gtrsim \sigma^2 \log(1/\delta) \cdot \omega(f,x^*,2\epsilon)^{-1}$ samples. The above argument can also be extended to the case where the shift $\tbad$ is chosen at random (as opposed to depending on the design).
\end{remark}
\end{document}